\documentclass[11pt]{article}

\usepackage{amssymb, amsmath, amsfonts,}
\usepackage{amsthm}
\usepackage{epsfig}
\usepackage{ifpdf}
\usepackage{epstopdf}
\usepackage[]{apacite}
\usepackage{subfigure}

\theoremstyle{plain}
\newtheorem{theorem}{Theorem}[section]

\theoremstyle{definition}
\newtheorem{definition}[theorem]{Definition}

\theoremstyle{remark}

\graphicspath{{images/}}
\begin{document}



\title{Preorder-Based Triangle: A Modified Version of Bilattice-Based Triangle for Belief Revision in Nonmonotonic Reasoning}

\author{Kumar Sankar Ray\\
ECSU,Indian Statistical Institute, Kolkata
\and
Sandip Paul\\
ECSU,Indian Statistical Institute, Kolkata
\and
Diganta Saha\\
CSE Department\\
Jadavpur University}

\maketitle

\begin{abstract}
Bilattice-based triangle provides an elegant algebraic structure for reasoning with vague and uncertain information. But the truth and knowledge ordering of intervals in bilattice-based triangle can not handle repetitive belief revisions which is an essential characteristic of nonmonotonic reasoning. Moreover the ordering induced over the intervals by the bilattice-based triangle is not sometimes intuitive. In this work, we construct an alternative algebraic structure, namely preorder-based triangle and we formulate proper logical connectives for this. It is also demonstrated that Preorder-based triangle serves to be a better alternative to the bilattice-based triangle for reasoning in application areas, that involve nonmonotonic reasoning and fuzzy reasoning with uncertain information.
\end{abstract}

\section{Introduction:}

In many application domains, decision making and reasoning deal with imprecise and incomplete information. Fuzzy set theory is a formalism for representation of imprecise, linguistic information. A vague concept is described by a membership function, attributing to all members of a given universe X a degree of membership from the interval [0,1]. The graded membership value refers to many-valued propositions in presence of complete information. But this 'one-dimensional' measurement cannot capture the uncertainty present in information. In absence of complete information the membership degree may not be assigned precisely. This uncertainty with respect to the assignment of membership degrees is captured by assigning a range of possible membership values, i.e. by assigning an interval. Interval-valued Fuzzy Sets (IVFSs) deal with vagueness and uncertainty simultaneously by replacing the crisp [0,1]-valued membership degree by intervals in [0,1]. The intuition is that the actual membership would be a value within this interval. The intervals can be ordered with respect to their degree of truth as well as with respect to their degree of certainty by means of a bilattice-based algebraic structure, namely Triangle  \cite {arieli2004relating,arieli2005bilattice,cornelis2007uncertainty}. This algebraic structure serves as an elegant framework for reasoning with uncertain and imprecise information.

The truth and knowledge ordering of intervals as induced by the bilattice-based triangle are inadequate for capturing the repetitive revision and modification of belief in nonmonotonic reasoning and are not always intuitive. In this paper we address this issue and attempt to propose an alternate algebraic structure to eliminate the shortcomings of bilattice-based triangle. The major contributions of this paper are as follows:

$\bullet$ We demonstrate, with the help of proper examples (in section 3), that bilattice-based triangle is incapable of handling belief revision associated with nonmonotonic reasoning. In nonmonotonic reasoning, inferences are modified as more and more information is gathered. The prototypical example is inferring a particular individual can fly from the fact that it is a bird, but retracting that inference when an additional fact is added, that the individual is a penguin. Such continuous belief revision is not properly represented in bilattice-based triangle.

$\bullet$ We point that the truth ordering is unintuitive regarding the ordering of intervals when one interval lies completely within the other (section 3)and hence not suitable for some practical applications.

$\bullet$ Exploiting the discrepancies mentioned, we propose modifications for knowledge ordering and truth ordering of intervals so that the aforementioned shortcomings are removed (in section 4).

$\bullet$ Using the modified knowledge and truth ordering we construct an alternate algebraic structure, namely preorder-based triangle (in section 5). This structure can be thought of as a unification of bilattice-based triangle and default bilattice \cite{ginsberg1988multivalued}. With this we come out of the realm of bilattice-based structures and explore a new algebraic structure based on simple linear pre-ordering.

$\bullet$ The proposed algebraic structure is then equipped with appropriate logical operators, i.e. negation, t-norms, t-conorms, implicators, in section 6. Most of the operators are in unison with those used for the bilattice-based structure. But the modified orderings offer additional flexibility.

$\bullet$ The proposed algebraic structure is shown to be capable of handling commonsense reasoning problems that could not be handled by the bilattice-based triangle (in section 7). Moreover, it is demonstrated that the preorder-based triangle can be employed to construct an answer set programming paradigm suitable for nonmonotonic reasoning with vague and uncertain information.

\section {Intervals as degree of belief:}

This section addresses some of the basic definitions and notions that will ease the discussion in the forthcoming sections.

Uncertainty and incompleteness of information is unavoidable in real life reasoning. Hence, sometimes it becomes difficult and misleading, if not impossible, to assign a precise degree of membership to some fuzzy attribute or to assert a precise degree of truth to a proposition. Therefore, assigning an interval of possible truth values is the natural solution. Intervals are appropriate to describe experts' degrees of belief, which may not be precise \cite{nguyen1997interval}. If an expert chooses a value, say 0.8, as his degree of belief for a proposition, actually we can only specify vaguely that his chosen value is around 0.8 and can be represented by an interval, say $[0.75, 0.85]$. Otherwise an interval can be thought of as a collection of possible truth values that a single or multiple rational experts would assign to a proposition in a scenario. Due to lack of complete knowledge the assertions made by different experts will be different and this lack of unanimity can be reflected by appropriate interval. The natural ordering of degree of memberships $(\leq)$ can be extended to the set of intervals and that gives rise to IVFS \cite{sambuc1975functions}.

An IVFS can be viewed as an L-fuzzy set \cite {goguen1967fuzzy} and the corresponding lattice can be defined as \cite{deschrijver2007bilattice}:
	
\begin{definition} 
$\textbf{L}^I=(L^I,\leq_L)$, \ where $L^I = \{[x_1,x_2]|(x_1,x_2)\in[0,1]\times[0,1]$ \ and \ $x_1\leq x_2\}$ and $[x_1,x_2] \leq_L[y_1,y_2]$ \ iff \ $x_1 \leq y_1$ \ and \ $x_2 \leq y_2$.
\end{definition}

\begin{figure}
\begin{center}
\includegraphics[width=50mm]{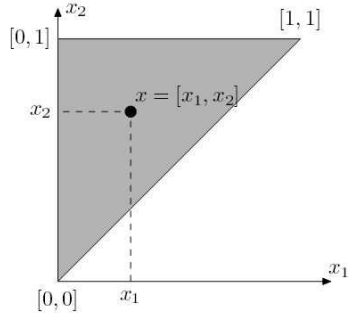}
\caption{The grey area is $L^I$}
\label{fig:closed_subintervals}
\end{center}
\end{figure}

In the definition, $L^I$ is the set of all closed subintervals in [0,1].Figure \ref{fig:closed_subintervals} shows the set $L^I$.

\textbf{Bilattice-based Triangle:}
Bilattices are ordered sets where elements are partially ordered with respect to two orderings, typically one depicts the degree of vagueness or truth (namely, truth ordering) and the other one depicting the degree of certainty (namely, knowledge ordering) \cite {arieli2004relating, cornelis2007uncertainty}. A bilattice-based triangle, or simply Triangle, can be defined as follows:

\begin {definition} \label{bilattdef}
Let \ $\textbf{L}=(L,\leq_L)$ \ be a complete lattice and let \ $I(L) = \{[x_1,x_2] \mid (x_1,x_2)\in \ L^2 $\ and $x_1\leq_L x_2\}$. A (bilattice-based) triangle is defined as a structure $\textbf{B(L)} = (I(L),\leq_t,\leq_k)$, where, for every $[x_1,x_2],[y_1,y_2]$ in $I(L)$:

1. $[x_1,x_2] \leq_t [y_1,y_2] \ \Leftrightarrow \ x_1 \leq_Ly_1$ and $x_2\leq_Ly_2$.

2. $[x_1,x_2] \leq_k [y_1,y_2] \ \Leftrightarrow \ x_1 \leq_Ly_1$ and $x_2\geq_Ly_2$.
\end{definition}
This triangle \textbf{B(L)} is not a bilattice, since, though the substructure $(I(L),\leq_t)$ is a complete lattice but $(I(L),\leq_k)$ is a complete semilattice.

When $\textit{L}$ is the unit interval [0,1], then I(L) describes membership of IVFS, $L^I$, and the lattice $\textbf{L}^I$ becomes $(I(L),\leq_t)$. In knowledge ordering, the intervals are ordered by set inclusion, as was suggested by Sandewall \cite {sandewall1989semantics}. The knowledge inherent in an interval $[c,d]$ is greater than another interval $[a,b]$ if $[c,d] \subseteq [a,b]$. 

\begin{figure}
\begin{center}
\includegraphics[width=60mm]{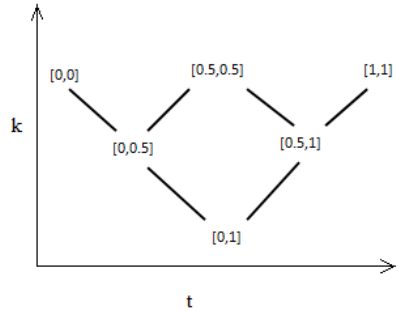}
\caption{Triangle B(\{0,0.5,1\})}
\label{fig:bitriangle}
\end{center}
\end{figure}

Triangle $\textbf{B}(\{0,0.5,1\})$ is shown in Figure \ref{fig:bitriangle}.

\section{Inadequacy of Bilattice-based Triangle:}

Intervals are used to approximate degree of truth of propositions in absence of complete knowledge. All values within an interval are considered to be equally probable to be the actual truth value of the underlying proposition. Thus considering intervals as truth status or epistemic state of propositions enables efficient representation of vagueness and uncertainty of information and reasoning. However, the Triangle structure suffers from the following shortcomings that must be eliminated.

\subsection{Inadequacy in modeling belief revision in nonmonotonic reasoning:} 
One important aspect of human commonsense reasoning is that it is nonmonotonic in nature \cite {brewka1991nonmonotonic}. In many cases conclusions are drawn in absence of complete information and we have to draw plausible conclusions based on the assumption that the world in which the reasoning is performed is normal and as expected. This is the best that can be done in contexts where the acquired knowledge is incomplete. But, these conclusions may have to be given up in light of further information. A proposition that was assumed to be true, may turn out to be false when new information is gathered. Such repetitive alterations of believes is an essential part of nonmonotonic reasoning. This type of belief revision may not be adequately represented by Triangle. The following discussion will illuminate this issue.

\subsubsection{An intuitive explanation:}
\label{sec:example1}
\textbf{Example 1:}Suppose the following information is given:
\vspace{0.15in}

Rules:

$Bird(x) \longrightarrow Fly(x)$,   [Birds Fly]

$Penguin(x) \longrightarrow \neg Fly(x)$,    [Penguin doesn't Fly]
\vspace{0.15in}

Facts:

Bird (Tweety)  [Tweety is a bird]

Given this information, suppose, multiple experts are trying to assess the degree of truth of the proposition "Tweety Flies" [Fly (Tweety)]. The rule "Birds Fly" is not a universally true fact, rather it's a general assumption that has several exceptions. Thus, being a Bird is not sufficient to infer that it will fly, since it may be a Penguin, an ostrich or some other non-flying bird. Since, nothing is specified about Tweety except for it is a bird, it is natural in human commonsense reasoning to "assume" that Tweety is not an exception and it will fly. Now, the confidence about this "asumption" will be different for different experts. An expert may bestow his complete faith on the fact that Tweety is not an exceptional bird and he will assign truth value 1 to "Tweety flies". Another expert may remain indecisive as he cannot ignore the chances that Tweety may be a non-flying bird and he will assign 0.5 (neither true nor false) to the proposition "Tweety flies". Others' assignments may be at some intermediate level depending on their perception about the world. Thus, the experts' truth assignments collectively construct an interval $[0.5,1]$ as the epistemic state of the rule "Birds fly" as well as of the fact " Fly(Tweety)".

Now, suppose an additional information is acquired that:

\begin{center}
Penguin(Tweety).   [Tweety is a penguin]
\end{center}

Then all the experts will unanimously declare Tweety doesn't fly and assign an interval $[0,0]$ as the revised epistemic state of the proposition "Tweety flies". The epistemic state of the proposition " Tweety flies" was first asserted by an interval $[0.5,1]$ and later the experts retracted their previously drawn decision to assert another interval $[0,0]$. From intuition it can be claimed that the interval $[0,0]$ makes a more confident and precise assertion than $[0.5,1]$, since in the former case all the experts were unanimous. But this is not reflected in the bilattice-based triangle (Figure \ref{fig:bitriangle}); since in Triangle $[0.5,1]$ and $[0,0]$ are incomparable in knowledge ordering. Thus, given the two intervals, based on the triangle structure, we remain clueless about which interval has higher degree of knowledge and which interval we should take up as final assertion of " Tweety flies". This is counter-intuitive and unwanted.

\begin{figure}
\begin{center}
\subfigure[Default Bilattice]{
\resizebox*{5cm}{!}{\includegraphics{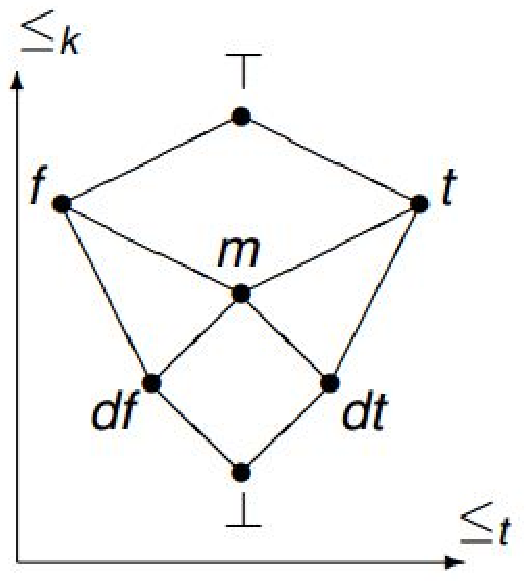}}}\hspace{5pt}
\subfigure[Multivalued Default Bilattice]{
\resizebox*{6cm}{!}{\includegraphics{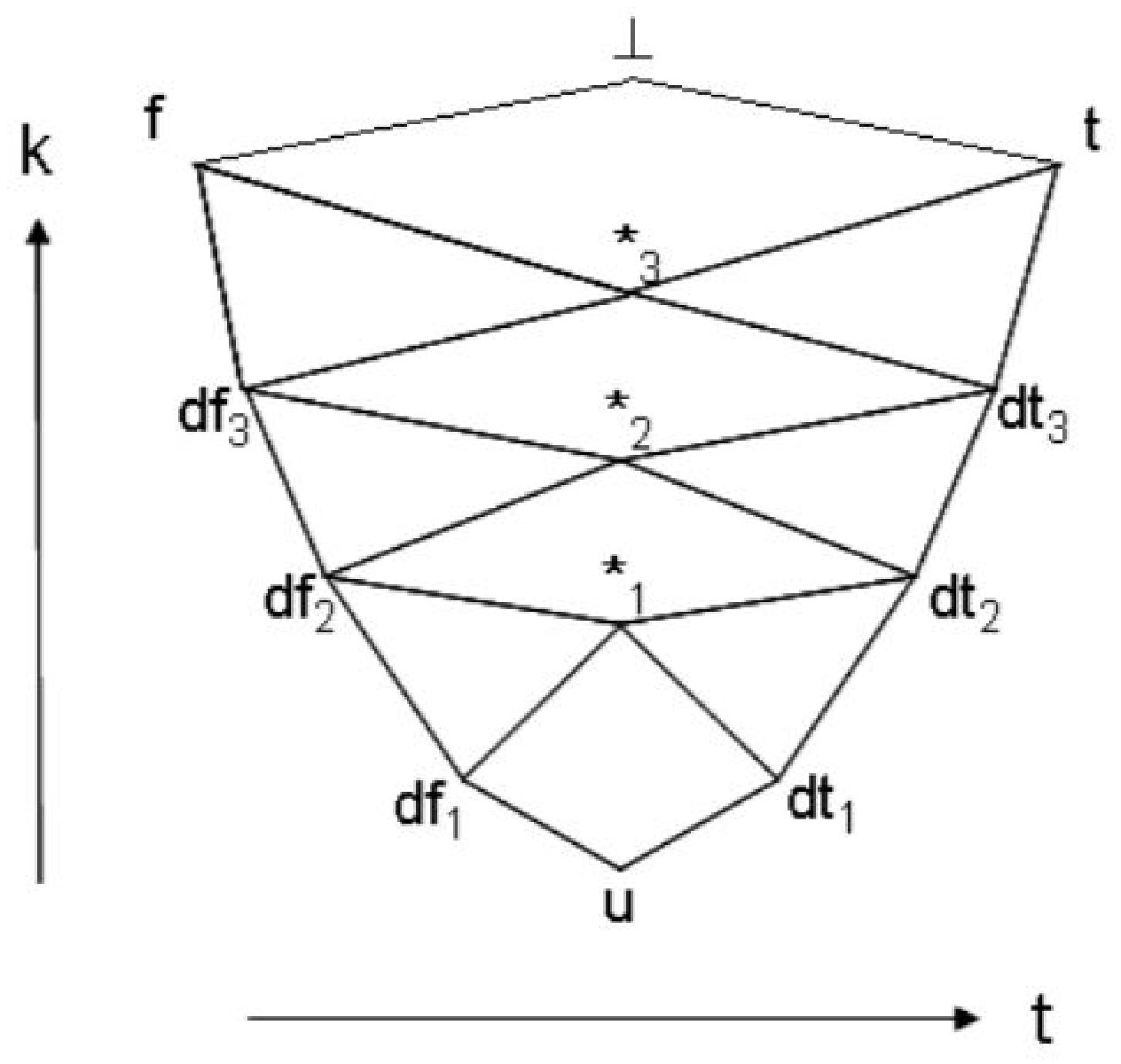}}}
\caption{Default bilattices for Nonmonotonic Reasoning} 
\label{fig:default}
\end{center}
\end{figure}

This type of scenario can be efficiently taken care of with the default bilattice \cite {ginsberg1988multivalued}. The general rule "Birds fly" will be assigned 'dt', i.e. true by default. Hence, 'Tweety flies' will also get dt. After acquiring the knowledge that Tweety is a penguin, 'Tweety flies' is asserted definitely false, i.e. f. In the default bilattice (Figure \ref{fig:default}.a) $f \geq_k dt$, expressing that the later conclusion is more certain than the earlier one.

The aforementioned example demonstrates that Triangle is incapable of depicting the continuous revision of decisions in absence of complete knowledge. Default bilattice is more appropriate than Triangle with respect to belief revision in nonmonotonic reasoning, but, vagueness or imprecision of information cannot be represented in Default bilattice.

\subsubsection{Example from an application domain:}
\label{sec:example2}
Bilattice-based structures are put to use for logical reasoning involving human detection and identity maintenance in visual surveillance systems by Shet et al. \cite{shet2007bilatticeth, shet2007bilattice, shet2006multivalued}. 

Multi-valued default bilattice (also known as prioritized default bilattice (Figure \ref{fig:default}.b) has been used for identity maintenance and contextual reasoning in visual surveillance system \cite{shet2006multivalued}. However in practice, logical facts are generated from vision analytics, which rely upon machine learning and pattern recognition techniques and generally have noisy values. Thus, in practical applications it would be more realistic to attach arbitrary amount of beliefs to logical rules rather than values such as dt, df etc that are allowed in multivalued default bilattices. For instance, similarity of different persons based on their appearances is a fuzzy attribute and may attain any degree over the [0,1] scale. But this cannot be captured by the multivalued default bilattice.

Bilattice-based square \cite{arieli2005bilattice} has been used for human detection in visual surveillance system. This algebraic structure is a better candidate than multivalued default bilattice as it provides continuous degrees of belief states.

The difference between bilattice-based square and bilattice-based triangle is that the former allows explicit representation of inconsistent information with different degrees of inconsistency. But it is pointed out by Dubois \cite{dubois2008ignorance} that square-like bilattices, where explicit representations of unknown $(\left\langle 0,0\right\rangle)$ and inconsistent $(\left\langle 1,1\right\rangle)$ epistemic states are allowed, can not preserve classical tautologies and sometimes give rise to unintuitive results.

Hence, bilattice-based triangle seems to be the most dependable and suitable algebraic structure to be used in the aforementioned applications.

Now let's apply the bilattice-based triangle to a slightly modified version of an example demonstrated by Shet et. al. \cite{shet2006multivalued, shet2006top} involving logical reasoning in identity maintenance. 

\textbf{Example 2:} The example deals with determining whether two individuals observed in an image should be considered as being one and the same. The rules and facts along with the assigned epistemic states are as follows:
 
rules:

r1: $\phi [appear\_similar (P_1, P_2) \longrightarrow equal (P_1, P_2)] = [0.5,1]$ 

r2: $\phi [distinct (P_1, P_2) \longrightarrow \neg equal (P_1, P_2)] = [0.9,1]$ 
\vspace{0.15in}

facts:

f1: $\phi [appear\_similar (a,b)] = [0.8,0.8]$

f2: $\phi [appear\_similar (c,d)] = [0.5,0.5]$

f3: $\phi [distinct (a,b)] = [1,1]$
\vspace{0.15in}

The specified set of facts depicts that individuals $a$ and $b$ are more similar than $c$ and $d$. Rules r1 and r2 encode the judgments of two different information sources or different algorithms, none of which present a confident, full-proof answer. However rule r2 (which may be based on some more accurate and highly reliable information) gives greater assurance to the non-equality of two persons than the assertion of equality expressed by rule r1, which may have came from a simple appearance matching technique of low dependability.

Intuitively, from the given information, a rational agent would put more confidence to the fact that individuals $a,b$ are not equal than on their equality; since the degree of distinction is more than the degree of similarity.The inference mechanism is specified in \cite {ginsberg1988multivalued}. The closure operator over the truth assignment function $\phi$ $(cl(\phi))$denotes the truth assignment that labels information entailed from the given set of rules and facts. The operator $cl_+(\phi)(q)$ takes into account set of rules that entail q and $cl_-(\phi)(q)$ considers set of rules that entail $\neg q$. Here the conjunctor, disjunctor and negator used are min, max and $1-$ operators respectively.

$cl_+(\phi)(equal(a,b))=[[0,1] \vee ([0.8, 0.8] \wedge [0.5,1])] = [[0,1] \vee [0.5,0.8]] = [0.5,0.8]$

$cl_-(\phi)(equal(a,b)) = \neg [[0,1] \vee ([1,1] \wedge [0.9,1])] = \neg [[0,1] \vee [0.9,1]] = [0,0.1]$
\vspace{0.15in}

Now the two intervals $[0.5,0.8]$ and $[0,0.1]$ are neither comparable with respect to $\leq_k$ in Triangle nor they have a $lub_k$ in the Triangle structure. Thus the two intervals cannot be combined to get a single assertion for $equal(a,b)$. Hence, using Triangle it is not possible to achieve the intended inference that $a$ and $b$ don't seem to be equal.

Thus the knowledge ordering in bilattice-based triangle must be modified in order to remove the aforementioned discrepancy. The modified knowledge ordering must incorporate within Triangle the ability to perform reasoning in presence of nonmonotonicity and demonstrate the repetitive belief revision, as the default bilattice has.

\subsection{Truth ordering is not always accurate:}
\label{sec:truthinaccurate}

An interval, taken as an epistemic state for a proposition, specifies the optimistic and pessimistic boundaries of the truth value of the proposition. In the bilattice-based triangle intervals are ordered with respect to $\leq_t$ ordering based on the boundaries of intervals; for two intervals $[x_1, x_2]$ and $[y_1, y_2]$, $[x_1, x_2] \leq_t [y_1, y_2]$ iff $x_1 \leq y_1$ and $x_2 \leq y_2$. Now with this ordering, any two intervals $x$ and $y$ are incomparable if $x$ is a proper sub-interval of $y$ or vice-versa, i.e. if one interval lies completely within the other with no common boundary. The justification behind this incomparability is that, if an interval, say $y$, is a proper sub-interval of $x$ then the actual truth value approximated by interval $x (\hat{x})$ may be greater or less than that of $\hat{y}$. For instance, if $x = [0.4, 0.8]$ and $y = [0.6, 0.7]$ then $\hat{x}$ can be less than $\hat{y}$ ( if $\hat{x} \in [0.4, 0.6)$) or $\hat{x}$ can be greater than $\hat{y}$ (if $\hat{x} \in (0.7,0.8]$). 

But similar situation may arise even when two intervals are not proper sub-interval of one another but just overlap, e.g. say $x=[0.4,0.8]$ and $y=[0.6,0.9]$. Yet these intervals are t-comparable, i.e., $[0.4, 0.8] \leq_t [0.6, 0.9]$. As the two intervals overlap, it is not ensured that the real truth value approximated by the lower interval will be smaller than the real truth value approximated by the higher interval (e.g. though $x \leq_t y$ but it may be the case that $\hat{x}=0.75$ and $\hat{y}=0.65$). In this respect the comparibility of these two intervals is not justified. Therefore, it is not always the most accurate ordering and can be regarded as a " weak truth ordering" \cite {esteva1994enriched}.

The intuitive justification in support for the truth ordering is given as \cite {deschrijver2009generalized}:
\begin{center}
"\textbf {$x \leq_t y$ iff the probability that $\hat{x} \leq \hat{y}$ is larger than $\hat{x} \geq \hat{y}$}"  \  \  \  \  \  \ $(*)$
\end{center}
i.e. the basic intuition behind truth ordering of two intervals $x$ and $y$ lies in comparing the probabilities $Prob (\hat{x} \geq \hat{y})$ and $Prob(\hat{x} \leq \hat{y})$, where, $\hat{x}$ and $\hat{y}$ are the actual truth values approximated by intervals $x$ and $y$ respectively.

Now, let us check whether this intuition holds good for the aforementioned pairs of intervals and eliminates the anomaly regarding their comparibility. Let us denote the three intervals $[0.4,0.8]$, $[0.6,0.7]$ and $[0.6,0.9]$ by $x,y,z$ respectively. Now for the pair of intervals $x$ and $y$, lets calculate the probabilities as specified in the right hand side of the iff condition in statement $(*)$. 

Since for an interval $[x_1,x_2]$ any value $\alpha \in [x_1,x_2]$ is equally probable to be equal to $\hat{x}$ (i.e. there is a uniform probability distribution over $[x_1, x_2]$) then for a sub-interval $[a,b]$ of $[x_1,x_2]$ we have, $Prob(\hat{x} \in [a,b]) = Prob(\hat{x} \in (a,b]) = Prob(\hat{x} \in [a,b)) = \frac{b-a}{x_2-x_1}$.

$\bullet$ For intervals $x$ and $y$, i.e. $[0.4,0.8]$ and $[0.6,0.7]$ respectively,

$Prob (\hat{x} \leq \hat{y})$

$=Prob (0.4\leq\hat{x}\leq 0.6$ and $ 0.6\leq\hat{y}\leq 0.7)+Prob (\hat{x} \leq \hat{y}|\hat{x},\hat{y} \in [0.6,0.7])$

$=Prob (0.4\leq\hat{x}\leq 0.6).Prob(0.6\leq\hat{y}\leq 0.7)+Prob(\hat{x}\leq\hat{y}|\hat{x},\hat{y}\in [0.6,0.7])$

$=\frac{(0.6-0.4)}{(0.8-0.4)}.1 + Prob(\hat{x} \leq \hat{y}|\hat{x},\hat{y} \in [0.6,0.7])$

$=0.5 + Prob(\hat{x} \leq \hat{y}|\hat{x},\hat{y} \in [0.6,0.7])$.\\

$Prob(\hat{x} \geq \hat{y})$

$=Prob (0.7 \leq \hat{x} \leq 0.8$ and $ 0.6 \leq \hat{y} \leq 0.7) + Prob (\hat{x} \geq \hat{y}|\hat{x},\hat{y} \in [0.6,0.7])$

$=Prob (0.7 \leq \hat{x} \leq 0.8).Prob(0.6 \leq \hat{y} \leq 0.7) + Prob (\hat{x} \geq \hat{y}|\hat{x},\hat{y} \in [0.6,0.7])$

$=\frac{(0.8-0.7)}{(0.8-0.4)}.1 + Prob(\hat{x} \geq \hat{y}|\hat{x},\hat{y} \in [0.6,0.7])$

$=0.25 + Prob(\hat{x} \geq \hat{y}|\hat{x},\hat{y} \in [0.6,0.7])$.

Now within the overlapped interval $[0.6,0.7]$, $\hat{x} \leq \hat{y}$ and $\hat{x} \geq \hat{y}$ are equally probable,i.e.
\begin{center}
$Prob (\hat{x} \leq \hat{y} | \hat{x}, \hat{y} \in [0.6,0.7]) = Prob (\hat{x} \geq \hat{y} | \hat{x}, \hat{y} \in [0.6,0.7])$. 
\end{center}

So, $Prob (\hat{x} \leq \hat{y}) > Prob (\hat{x} \geq \hat{y})$.

$\bullet$ For intervals $x$ and $z$, i.e. $[0.4,0.8]$ and $[0.6,0.9]$ respectively,

$Prob (\hat{x} \leq \hat{z})$

$=Prob (0.4 \leq \hat{x} \leq 0.6$ and $ 0.6 \leq \hat{z} \leq 0.9) + Prob (0.6 \leq \hat{x} \leq 0.8$ and $ 0.8 \leq \hat{z} \leq 0.9)+ Prob (\hat{x} \leq \hat{z}|\hat{x},\hat{z} \in [0.6,0.8])$

$=Prob (0.4 \leq \hat{x} \leq 0.6).Prob(0.6 \leq \hat{z} \leq 0.9) + Prob (0.6 \leq \hat{x} \leq 0.8).Prob(0.8 \leq \hat{z} \leq 0.9) + Prob (\hat{x} \leq \hat{z}|\hat{x},\hat{z} \in [0.6,0.8])$

$=\frac{(0.6-0.4)}{(0.8-0.4)}.1 + \frac{(0.8-0.6)}{(0.8-0.4)}.\frac{(0.9-0.8)}{(0.9-0.6)} + Prob(\hat{x} \leq \hat{z}|\hat{x},\hat{z} \in [0.6,0.8])$

$=0.5 + 0.167 + Prob(\hat{x} \leq \hat{z}|\hat{x},\hat{z} \in [0.6,0.8])$.

$=0.667 + Prob(\hat{x} \leq \hat{z}|\hat{x},\hat{z} \in [0.6,0.8])$.\\

$Prob(\hat{x} \geq \hat{z})$

$=Prob (\hat{x} \geq \hat{z}|\hat{x},\hat{z} \in [0.6,0.8])$

Again within the overlapped portion $[0.6,0.8]$, $\hat{x} \leq \hat{z}$ and $\hat{x} \geq \hat{z}$ are equally probable,i.e.
\begin{center}
$Prob (\hat{x} \leq \hat{z} | \hat{x}, \hat{z} \in [0.6,0.8]) = Prob (\hat{x} \geq \hat{z} | \hat{x}, \hat{z} \in [0.6,0.8])$. 
\end{center}

So, $Prob (\hat{x} \leq \hat{z}) > Prob (\hat{x} \geq \hat{z})$.

Thus it can be seen for pairs of intervals $x,y$ and $x,z$:
\begin{center}
$Prob (\hat{x} \leq \hat{y}) \geq Prob(\hat{x} \geq \hat{y})$ and\\
$Prob (\hat{x} \leq \hat{z}) \geq Prob(\hat{x} \geq \hat{z})$
\end{center}

Thus we can see that from the probabilistic perspective the two pairs of intervals behave similarly but the truth ordering $\leq_t$ treats them differently. Following the intuition of truth ordering, as depicted in statement $(*)$, both pairs of intervals should be comparable with respect to $\leq_t$ ordering and it should have been the case that $x \leq_t y$ and $x \leq_t z$. However, surprisingly although we have $x \leq_t z$ (since $x_1 < z_1$ and $x_2 < z_2$) but intervals $x$ and $y$ are not comparable with respect to the truth ordering $\leq_t$ (since $x_1 < y_1$ and $x_2 > y_2$). 

Thus the truth ordering in bilattice-based triangle is not intuitive when it compares partially or completely overlapped intervals.  The incomparability of intervals $x$ and $z$ with respect to $\leq_t$ contradicts with the intuition of truth ordering$(\leq_t)$. This may generate unintuitive and problematic results in application areas, specially when reasoning is done in absence of complete knowledge.

\vspace{0.15in}
\textbf{Application area where truth ordering fails to perform reasoning:}

Artificial intelligence based systems are proposed to be used for medical diagnosis and artificial triage system in emergency wards \cite{golding2008emergency, burke1990artificial, wilkes2010heterogeneous}. Use of possibilistic answer set programming for medical diagnosis has been reported in \cite{bauters2014semantics}. 

Possibilistic approach is suitable for capturing the uncertainty present in information. For instance, in the framework we can represent and reason with the possibility of a particular disease, given a set of symptoms. But sometimes the severity level of the particular disease,i.e. whether it is in primary stage or advanced stage, is also essential to know. For instance, only knowing a patient has coronary blockage is not sufficient, but whether there is any risk of heart attack or what type of surgery is to be done is dependent on the \textit{percentage of blockage} of the coronary artery. Again if a patient comes in an emergency ward with stomach ache then the triage nurse would assess the urgency depending on the degree of pain the patient is experiencing. Sometimes the patient is asked to rate his/her pain on a scale of 1 to 10. These are not uncertain quantities but fuzzy quantities. Therefore the truth value of the statement "Patient A is suffering from disease X" would not be bivalent rather would have to be chosen from a continuous range of values from $[0,1]$. This information can not be represented in the Possibilistic ASP, which is essentially based on two valued logic.

In practice, several factors contribute to the severity of a particular ailment. Therefore sometimes it becomes difficult to diagnose a disease with a specific severity degree; rather it is more natural to ascribe a range of values.This may be due to the presence of subjective uncertainties, e.g. fluctuating blood pressure or body temperature or nonspecific rating of a patient's pain. Another reason may be incomplete knowledge, based on which decisions are being taken. For example, in an emergency situation medical decisions have to be taken rapidly when the doctors can't wait for all the test results and must make their decisions based on assumptions, rules of thumb and experience. Again, multiple agents (in this case doctors in medical board or different triage nurses) may not be unanimous about a judgment. Thus, the epistemic state of a statement like 'Patient A has disease d1' would be comprised of an subinterval in $[0,1]$, where each individual value in the interval represent an expert's opinion regarding how true the statement is. The epistemic state of the statement 'Patient A has disease d1' is same as the suspected severity degree of the disease. When the experts have complete knowledge and there isn't any uncertainty, they provide unanimous opinion depicted by an exact interval of the form $[x,x]$, where the value $x$ would signify the severity of the disease. More is the uncertainty wider would be the assigned interval. For example, the pain rated by a patient by the range $5-7$ on a scale of $0-10$, can be represented by the interval $[0.5, 0.7]$. Therefore when a system is built to perform reasoning with these types of information subintervals of $[0,1]$ are assigned to simple propositions as their epistemic states denoting the vagueness and uncertainty. For example, the pain rated by a patient by the range $5-7$ on a scale of $0-10$, can be represented by the interval $[0.5, 0.7]$. Decisions regarding the prescribed treatment would be based on the epistemic states and comparing them. 

Bilattice-based triangle structure can be used to reason with such vague and uncertain information. The two natural orderings in bilattice-based triangle, namely knowledge ordering and truth ordering, can be employed to compare respectively certainty and vagueness about different propositions, e.g. severity degree of ailments. In medical decisions the truth ordering may play a more crucial role.

\textbf{Example 3.} Suppose an intelligent triage system has diagnosed a patient with two diseases (say $di1$ and $di2$). There may be situations where the doctor has to prioritize the treatment of the two diseases based on their severity. This situation may arise when medications for two diseases are mutually incompatible or both of them require surgery that can not be done simultaneously. Thus deciding which treatment is more urgent becomes crucial. 

Suppose the artificial triage system has the following rules in its system regarding the two diseases di1 and di2 and its medications (dr1 and dr2):

1. $\phi [di1,$ not $dr2 \longrightarrow dr1] = [1,1],$

2. $\phi [di2,$ not $dr1 \longrightarrow dr2] = [1,1],$

These rules denote the mutual incompatibility of drugs for the two diseases. They suggest that if a patient is diagnosed with disease 1(2) and drug 2(1) is not being administered then drug 1(2) can be given to the patient. Hence when a patient is diagnosed with both of these diseases the triage system must calculate which of the two diseases is more severe and requires urgent medication. Based on the symptoms the severity of the diseases are specified by the intervals $i_{d1}$ and $i_{d2}$ respectively. If the assigned epistemic states, that denotes the level of severity, are exact intervals then simply comparing the assigned value would be enough. But when $i_{d1}$ and $i_{d2}$ are non-exact intervals then it must be decided which among the knowledge ordering and truth ordering has to be used for comparing the intervals to compare the severity. An example can help to clarify the intuition. Suppose, the intervals $i_{d1}$ and $i_{d2}$ are respectively $[0.5,0.6]$ and $[0.6,0.9]$.  Here, as a whole $di2$ is suspected to be in a more severe condition than $di1$, because for each value $a \in i_{d2}$, $a$ is greater than the upper limit of $i_{d1}$. Hence though the uncertainty regarding $di2$ is more (since the interval $i_{d2}$ is wider than that of $i_{d1}$), but the chance that $di2$ is more critical is higher than that of $di1$ and $di2$ is to be treated before $di1$. The truth ordering in the bilattice-based triangle is supposed to capture this intuition and indeed in this case we have $i_{d2} \leq_t i_{d1}$ (since, $[0.5,0.6] \leq_t [0.6,0.9]$). Thus the triage system must have the following rules to prioritize the treatment considering the urgency or severity of the diseases:

3. $\phi [(\phi[di2] \leq_t \phi[di1]) \longrightarrow dr1] = [1,1]$

4. $\phi [(\phi[di1] \leq_t \phi[di2]) \longrightarrow dr2] = [1,1]$

Now suppose. based on the symptoms and test results (which may not be complete) the severity of each disease are asserted and fed into the triage system's database as follows:  

5. $\phi [di1] = [0.4,0.9]$

6. $\phi [di2] = [0.5,0.6]$

Now, considering the input information the triage system would have to decide which of the drugs among $dr1$ and $dr2$ have to be administered based on the severity. 

But now suppose $i_{d1}$ and $i_{d2}$ are respectively $[0.4,0.9]$ and $[0.5,0.6]$. These two intervals cannot be compared with respect to $\leq_t$. Therefore, in this situation, the artificial triage system cannot decide which disease is to be treated first. But this is not intuitive, since, following the similar analysis as demonstrated in the previous subsection it can be proved that it is more probable that the actual severity of $di1$ (as approximated by the interval $[0.4,0.9]$) is higher than that of $di2$ (as approximated by interval $[0.5,0.6]$). Thus, treatment of $di1$ is more urgent and hence $dr1$ must be administered first. 

The bilattice-based triangle is incapable of capturing this notion. Thus the truth ordering in bilattice-based triangle must be modified so that the new ordering behaves as $\leq_t$ when intervals are non-overlapping and captures the probabilistic essence stated in statement $(*)$ when intervals are overlapping.

\section {Modification in Triangle structure:}

Based on the discussions in the above two subsections the bilattice-based triangle is modified.

\subsection {Modification in knowledge ordering:}

The knowledge ordering can be defined based on just the length of intervals and irrespective of the real truth values they attempt to approximate. Thus for two intervals $[x_1, x_2]$ and $[y_1, y_2] \in L^I$, where, $L^I$ is the set of sub-intervals of $[0,1]$ as shown in Figure \ref{fig:closed_subintervals}

\begin{center}
$[x_1,x_2] \leq_{k_p} [y_1,y_2]\Leftrightarrow (x_2 - x_1) \geq (y_2-y_1)$.
\end{center}

that is, wider the interval lesser is the knowledge content. Equality of the width of intervals is a necessary condition for $x=y$, but not a sufficient condition; because two different intervals may have equal width, e.g. $[0.1,0.2]$ and $[0.7,0.8]$.

\begin{figure}
\begin{center}
\includegraphics [width = 80mm]{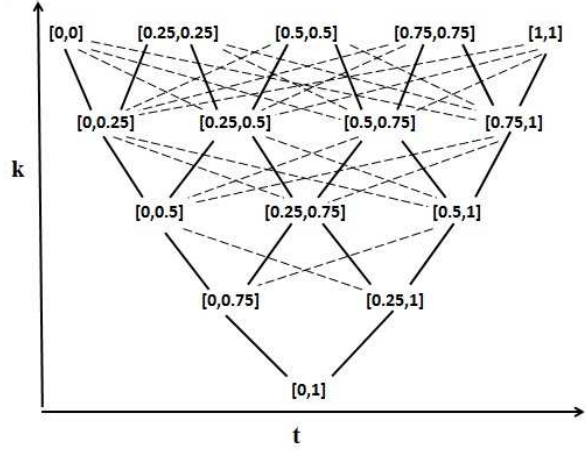}
\caption{I(\{0,0.25,0.5,0.75,1\}) with modified knowledge ordering}
\label{fig:modiknowledge}
\end{center}
\end{figure}

The algebraic structure for $(I(\{0,0.25,0.5,0.75,1\}), \leq_t, \leq_{k_p})$ is shown in Fig. \ref{fig:modiknowledge}.

\subsection {Modification in Truth Ordering:}

The truth ordering $(\leq_t)$ gives rise to certain discrepancies in ordering intervals, as discussed in section 3.2. Lets take statement $(*)$ as a starting point to revisit the truth ordering, especially in case when one interval is a proper sub-interval of the other. In this respect the following theorem is stipulated.
\vspace{0.1in}

\begin{theorem}

For two intervals $x=[x_1, x_2]$ and $y=[y_1,y_2] \in L^I$,
\begin{center}
$Prob (\hat{x} \geq \hat{y}) \leq Prob (\hat{x} \leq \hat{y}) \equiv x_m \leq y_m$
\end{center}
where, $\hat{x}(\hat{y})$ stands for the actual truth value approximated by the interval $x(y)$; and $x_m$ and $y_m$ are respectively the midpoints of intervals x and y.

\end{theorem}

\begin{proof}
The proof is constructed by considering several cases depending on how intervals $x$ and $y$ are situated on the $[0,1]$ scale. Without loss of generality it is assumed that $x_1 \leq y_1$ for showing the proof. For the other case, i.e. $x_1 > y_1$ similar proof can be constructed which is not shown here.
 
 Since any $x \in [x_1,x_2]$ is equally probable to be equal to $\hat{x}$ (i.e. there is a uniform probability distribution over $[x_1, x_2]$) then for a sub-interval $[a,b]$ of $[x_1,x_2]$ we have, $Prob(\hat{x} \in [a,b]) = Prob(\hat{x} \in (a,b]) = Prob(\hat{x} \in [a,b)) = \frac{b-a}{x_2-x_1}$.
\vspace{0.1in}

\textbf{Case 1:}

\begin{figure}
\begin{center}
\includegraphics[width=40mm]{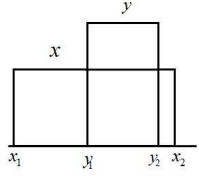}
\caption{y is a proper sub-interval of x}
\label{fig:subinterval}
\end{center}
\end{figure}

Suppose, $x=[x_1,x_2]$ has $y=[y_1,y_2]$ as a proper sub-interval (Figure \ref{fig:subinterval}). For these intervals $x_1 < y_1$ and $y_2 < x_2$, hence $x$ and $y$ can not be ordered using $\leq_t$. 

In this case,

1. $\hat{x} \leq \hat{y}$ iff $\hat{x} \in [x_1, y_1)$ or $(\hat{x} \leq \hat{y}$ given $\hat{x}, \hat{y} \in [y_1,y_2])$,

2. $\hat{x} \geq \hat{y}$ iff $\hat{x} \in (y_2, x_2]$ or $(\hat{x} \geq \hat{y}$ given $\hat{x}, \hat{y} \in [y_1,y_2])$.

Within the smaller interval $[y_1,y_2]$ the $\hat{x} \leq \hat{y}$ and $\hat{x} \geq \hat{y}$ are equally probable,i.e.

\begin{center}
$Prob (\hat{x} \leq \hat{y} | \hat{x}, \hat{y} \in [y_1,y_2]) = Prob (\hat{x} \geq \hat{y} | \hat{x}, \hat{y} \in [y_1,y_2])$. 
\end{center}

Now,

$Prob(\hat{x} \geq \hat{y}) \leq Prob(\hat{x} \leq \hat{y})$

$\equiv Prob(\hat{x} \in (y_2, x_2]$ or $(\hat{x} \geq \hat{y}$ given $\hat{x}, \hat{y} \in [y_1,y_2])) \leq Prob(\hat{x} \in [x_1, y_1)$ or $(\hat{x} \leq \hat{y}$ given $\hat{x}, \hat{y} \in [y_1,y_2]))$

$\equiv Prob(\hat{x} \in (y_2, x_2]) + Prob (\hat{x} \geq \hat{y} | \hat{x}, \hat{y} \in [y_1,y_2]) \leq Prob(\hat{x} \in [x_1, y_1)) + Prob (\hat{x} \leq \hat{y} | \hat{x}, \hat{y} \in [y_1,y_2])$

$\equiv Prob(\hat{x} \in (y_2, x_2]) \leq Prob(\hat{x} \in [x_1, y_1))$

$\equiv \frac{x_2 - y_2}{x_2 - x_1} \leq \frac{y_1 - x_1}{x_2 - x_1}$

$\equiv (x_2 - y_2) \leq (y_1 - x_1)$  (since $(x_2 - x_1) > 0$)

$\equiv (x_1 + x_2) \leq (y_1 + y_2)$

$\equiv \frac{x_1 + x_2}{2} \leq \frac{y_1 + y_2}{2}$

$\equiv$ the midpoint of interval $x$ $\leq$ the midpoint of interval $y$

$\equiv x_m \leq y_m$.
\vspace{0.1in}

\textbf{Case 2:}

\begin{figure}
\begin{center}
\includegraphics[width=40mm]{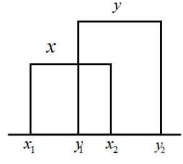}
\caption{x and y are overlapping}
\label{fig:overlapping}
\end{center}
\end{figure}

Suppose two intervals $x = [x_1, x_2]$ and $y = [y_1, y_2]$ are overlapping, as shown in Figure \ref{fig:overlapping}. In this case, $x_1 \leq y_1$ and $x_2 \leq y_2$.
\vspace{0.1in}

Here,

1. $\hat{x} \leq \hat{y}$ iff $\hat{x} \in [x_1, y_1)$ or $(\hat{x} \in [y_1,x_2]$ and $\hat{y} \in (x_2, y_2])$ or $(\hat{x} \leq \hat{y}$ given $\hat{x}, \hat{y} \in [y_1,x_2])$,

2. $\hat{x} \geq \hat{y}$ iff $(\hat{x} \geq \hat{y}$ given $\hat{x}, \hat{y} \in [y_1,x_2])$.
\vspace{0.15in}

$ Prob(\hat{x} \geq \hat{y}) \leq Prob(\hat{x} \leq \hat{y})$

$\equiv Prob (\hat{x} \geq \hat{y} | \hat{x}, \hat{y} \in [y_1,x_2]) \leq Prob(\hat{x} \in [x_1, y_1)) + Prob(\hat{x} \in [y_1,x_2]$ and $\hat{y} \in (x_2, y_2]) + Prob(\hat{x} \leq \hat{y} | \hat{x}, \hat{y} \in [y_1, x_2])$

$\equiv Prob(\hat{x}) \in [x_1, y_1) + Prob (\hat{x} \in [y_1,x_2]$ and $\hat{y} \in (x_2, y_2]) \geq 0$

\begin {flushright}
(since, $Prob (\hat{x} \geq \hat{y} | \hat{x}, \hat{y} \in [y_1,x_2]) = Prob(\hat{x} \leq \hat{y} | \hat{x}, \hat{y} \in [y_1, x_2])$)
\end{flushright}

$\equiv \frac{y_1 - x_1}{x_2 - x_1} + \frac {x_2 - y_1}{x_2-x_1}\frac{y_2 - x_2}{y_2 - y_1} \geq 0$

$\equiv (y_1 - x_1)(y_2 - y_1) + (y_2 - x_2)(x_2 - y_1) \geq 0$

$\equiv y_1y_2 - x_1y_2 - y_1^2 +x_1y_1 + x_2y_2 - x_2^2 - y_1y_2 + x_2y_1 \geq 0$

$\equiv x_1y_1 + x_2y_2 + x_2y_1 - x_1y_2 - y_1^2 - x_2^2 \geq 0$

\begin{flushright}
(cancelling $y_1y_2$ and $-y_1y_2$ and rearranging terms)
\end{flushright}

$\equiv x_2(y_2 + y_1 - x_2) - y_1(y_1 - x_1) - x_1y_2 \geq 0$

$\equiv x_2(y_2 + y_1 - x_2) - y_1(y_1 - x_1) - x_1x_2 \geq 0$

\begin{flushright}
(since $x_2 \leq y_2$)
\end{flushright}

$\equiv x_2(y_2 + y_1 - x_1 - x_2) - y_1(y_1 - x_1) \geq 0$

$\equiv x_2(y_2 + y_1 - x_1 - x_2) \geq 0$

\begin{flushright}
(since $x_1 \leq y_1$)
\end{flushright}

$\equiv (y_2 + y_1 - x_1 - x_2) \geq 0$

$\equiv (y_1 + y_2) \geq (x_1 + x_2)$

$\equiv \frac{y_1 + y_2}{2} \geq \frac{x_1 + x_2}{2}$

$\equiv$ the midpoint of interval $y$ $\geq$ the midpoint of interval $x$

$\equiv x_m \leq y_m$.

\textbf {Case 3:}

\begin{figure}
\begin{center}
\includegraphics[width=80mm]{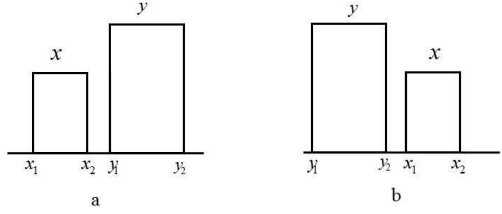}
\caption{x and y are disjoint and $x \leq_t y$}
\label{fig:disjoint}
\end{center}
\end{figure}

We can have two subcases for disjoint intervals (Figure \ref{fig:disjoint}). For subcase a, the interval $x$ is lower than the interval $y$, i.e. $\forall a \in [x_1, x_2], a \leq y_1$ or in other words $x_1 < x_2 \leq y_1 < y_2$. Similarly, for subcase b, the interval $y$ is lower than the interval $x$, i.e. $\forall b \in [y_1, y_2], b \leq x_1$ or in other words $y_1 < y_2 \leq x_1 < x_2$.
\vspace{0.1in}

In this case, since intervals are disjoint,

$Prob(\hat{x} \leq \hat{y}) = 1$ and $Prob(\hat{x} \geq \hat{y}) = 0$ if $x_2 \leq y_1$ (Subcase a);

$Prob(\hat{x} \leq \hat{y}) = 0$ and $Prob(\hat{x} \geq \hat{y}) = 1$ if $y_2 \leq x_1$ (Subcase b);
\vspace{0.1in}

Now,

$Prob(\hat{x} \geq \hat{y}) < Prob(\hat{x} \leq \hat{y})$

$\Rightarrow Prob(\hat{x} \geq \hat{y}) = 0$ and $Prob(\hat{x} \leq \hat{y}) = 1$

$\Rightarrow \forall a \in [x_1, x_2], a \leq y_1$

$\Rightarrow x_2 \leq y_1$

$\Rightarrow x_1 + x_2 \leq y_1 + x_1$

$\Rightarrow x_1 + x_2 < y_1 + y_2$  [since, $x_1 < y_1$]

$\Rightarrow x_m < y_m$.
\vspace{0.1in}

Again;
\vspace{0.15in}

$x_m < y_m$

$\Rightarrow x_1 + x_2 < y_1 + y_2$

$\Rightarrow x_1 < y_1$ and $x_2 < y_2$ and $x_2 \leq y_1$ [since intervals are disjoint]

$\Rightarrow Prob(\hat{x} \geq \hat{y}) = 0$ and $Prob(\hat{x} \leq \hat{y}) = 1$

$\Rightarrow Prob(\hat{x} \geq \hat{y}) < Prob(\hat{x} \leq \hat{y})$.
\vspace{0.1in}

Thus $Prob(\hat{x} \geq \hat{y}) < Prob(\hat{x} \leq \hat{y}) \equiv x_m < y_m$.

Here the proof ends.
\end{proof}

\vspace{0.1in}

Hence, it is proved that the straightforward way to compare the probabilities $Prob(\hat{x} \geq \hat{y})$ and $Prob(\hat{x} \leq \hat{y})$ for two intervals $x$ and $y$ is to compare their midpoints. Case 1 in the above proof is particularly interesting, where one interval is a proper sub-interval of the other. Though the chosen intervals $x$ and $y$ are not comparable with respect to $\leq_t$ ordering, but we can compare their midpoints and thus order the probabilities $Prob(\hat{x} \geq \hat{y})$ and $Prob(\hat{x} \leq \hat{y})$. Thus following statement $(*)$ a truth ordering can be imposed on $x$ and $y$ based on the probabilistic comparison. The existing truth ordering $(\leq_t)$ as shown in Definition \ref{bilattdef}, doesn't allow this comparability of $x$ and $y$, and hence a new truth ordering is called for.

Now that we are able to estimate and order the probabilities, in light of statement $(*)$ we are in a place to construct a generalised truth ordering $(\leq_{t_p})$ as follows:

\begin{center}
$x \leq_{t_p} y \Leftrightarrow x_m \leq y_m$.
\end{center}

\begin{figure}
\begin{center}
\includegraphics[width=40mm]{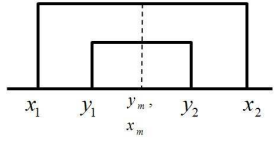}
\caption{Intervals incomparable in t-ordering but not equal}
\label{fig:incomparable}
\end{center}
\end{figure}

The equality of midpoints of two intervals $x$ and $y$, (i.e. $\frac{x_1 + x_2}{2} = \frac{y_1 + y_2}{2}$) is a \textit{necessary} condition for $x = y$, but not a \textit{sufficient} condition; because two different intervals can have same midpoint, as shown in Figure \ref{fig:incomparable}.

Moreover, the discrepancy mentioned in section 3.2 is resolved, since cases where intervals are overlapped and when one interval is a proper sub-interval of the other are treated uniformly and in each case intervals are comparable with respect to $\leq_{t_p}$.

\begin{theorem}

For two intervals $x = [x_1,x_2]$ and $y = [y_1,y_2] \in L^I$, such that none is a proper subinterval of the other,
\begin{center}
$x \leq_t y \Rightarrow x \leq_{t_p} y$.
\end{center}
\end{theorem}

\begin{proof}
From the definition,
\vspace{0.1in}

$x \leq_t y \Leftrightarrow x_1 \leq x_2$ and $y_1 \leq y_2$

 \  \  \  \  \  \  \  \  \  \ $\Rightarrow x_1 + x_2 \leq y_1 + y_2$

 \  \  \  \  \  \  \  \  \  \ $\Rightarrow x_m \leq y_m$
 
 \  \  \  \  \  \  \  \  \  \ $\Rightarrow x \leq_{t_p} y$.

\end{proof}

Thus, the probabilistic analysis gives a broader truth ordering of the intervals that can be achieved by comparing midpoints of intervals. For each pair of intervals if they are comparable with respect to $\leq_t$ they are also comparable with respect to the modified truth ordering $\leq_{t_p}$ and additionally $\leq_{t_p}$ can order intervals when one of them is a proper sub-interval of the other and hence are not $\leq_t-$comparable.

For instance, for two intervals $x = [0,1]$ and $y = [0.8,0.9]$ we have $[0,1] <_{t_p} [0.8, 0.9]$ though $x$ and $y$ are not t-comparable w.r.t. $\leq_t$. 

\begin{figure}
\begin{center}
\includegraphics [width=80mm]{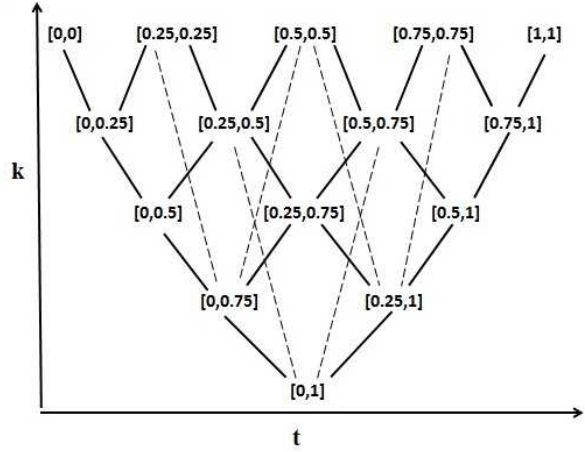}
\caption{I(\{0,0.25,0.5,0.75,1\}) with modified truth ordering}
\label{fig:modtruth}
\end{center}
\end{figure}

The algebraic structure for $(I(\{0,0.25,0.5,0.75,1\}), \leq_{t_p}, \leq_k)$ is shown in Figure \ref{fig:modtruth}.

\section {An alternative algebraic structure}

Based on these modifications we propose a modified and more intuitive algebraic structure for ordering intervals with respect to degree of truth and knowledge (or certainty).
\vspace{0.1in}

\textbf{Notation:} For an interval $x \in L^I$; $x_m$ and $x_w$ will be used to denote the midpoint (or center) and the length of the interval respectively; i.e. $x_m = (x_1+x_2)/2$ and $x_w = (x_2-x_1)$. The pair $(x_m, x_w)$ uniquely specifies an interval $x$ and hence may be used instead of the traditional representation $[x_1,x_2]$.
 
\begin{definition} \label{defmodi}
A preorder-based triangle is a structure $\textbf{P(L)} = (L^I, \leq_{t_p} , \leq_{k_p})$, defined for every $[x_1,x_2]$ and $[y_1,y_2] \in L^I$ as:
\\ 1. $[x_1,x_2] \leq_{t_p}[y_1,y_2]\Leftrightarrow x_m \leq y_m$,
\\ 2. $[x_1,x_2] \leq_{k_p}[y_1,y_2]\Leftrightarrow x_w \geq y_w$,
\\ 3. $ x = y \Leftrightarrow x_m = y_m$ and $x_w = y_w$.
\end{definition}  

\begin{figure}
\begin{center}
\includegraphics[width=70mm]{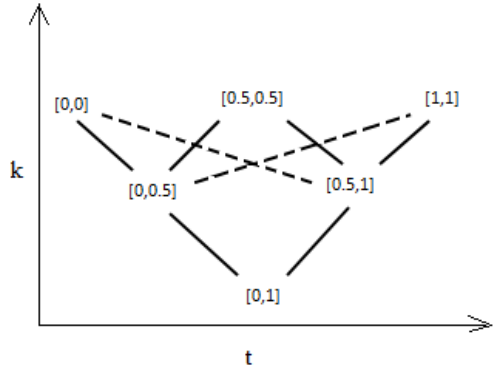}
\caption{Modified Triangle for I(\{0,0.5,1\}) }
\label{fig:smallmoditri}
\end{center}
\end{figure}

\begin{figure}
\begin{center}
\includegraphics [width=80mm]{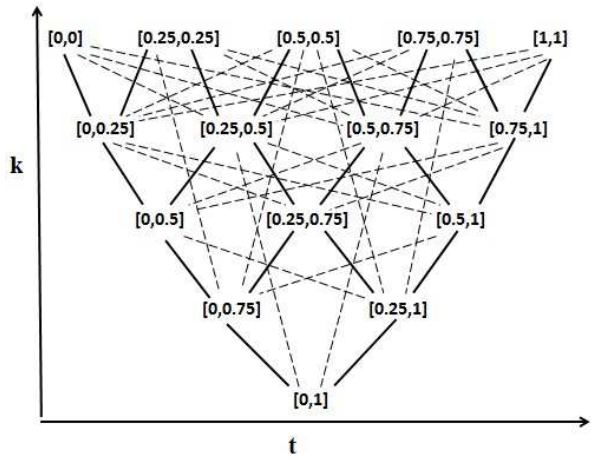}
\caption{Modified Triangle for I(\{0,0.25,0.5,0.75,1\})}
\label{fig:largemoditri}
\end{center}
\end{figure}

Preorder-based triangle can be defined for any subset of $L^I$ as well. For instance preorder-based triangle for $I(\{0,0.5,1\})$ and for $I(\{0,0.25,0.5,0.75,1\})$ are shown in Figure \ref{fig:smallmoditri} and Figure \ref{fig:largemoditri} respectively. The dashed lines demonstrate the connections that were absent in the bilattice-based triangle.

With the truth and knowledge ordering presented in Definition \ref{defmodi} we step out of the realm of lattice-based structures. The substructure $(L^I,\leq_{t_p})$ is not a lattice since for any two intervals $a$ and $b$, existence of $lub_{t_p}(a,b)$ and $glb_{t_p}(a,b)$ are not guaranteed. For instance, suppose $L=\{0,0.1,0.2..,1\}$ i.e. the unit interval discretised with eleven equidistant points. Now, two intervals in $I(L)$, $[0.8,0.8]$ and $[0.6,1]$ are incomparable with respect to $\leq_{t_p}$. The upper bound of the two intervals is not a unique element, but a set of intervals $\{[0.7,1],[0.8,0.9]\}$. Hence $lub_{t_p}$ doesn't exist. Lower bound of the two intervals is the set $\{[0.7,0.8], [0.6,0.9], [0.5, 1]\}$.

The ordering $\leq_{t_p}$ and $\leq_{k_p}$ over the set of intervals are not lattice-orders but \textbf{pre-orders}, i.e reflexive and transitive. The orderings are not symmetric clearly. Nor they are anti-symmetric, since $[0.5,0.5] \leq_{t_p} [0,1]$ and $[0,1] \leq_{t_p} [0.5,0.5]$ but $[0.5,0.5] \neq [0,1]$ and similar example holds for the k-ordering. Thus the substructure $(L^I,\leq_{t_p})$ and $(L^I,\leq_{k_p})$ form pre-ordered sets instead of lattices. Moreover, in $(L^I, \leq_{k_p})$ any set of intervals have lower bound but may not have upper bound. For instance, following the earlier example, intervals $[0.5,0.5]$ and $[0.8,0.8]$ doesn't have an upper bound, but has the set of intervals of length 0.1 as its lower bound.

Because of the modified knowledge ordering, the preorder-based triangle can be thought of as a unification of the default bilattice and the bilattice-based triangle.

One point must be emphasized here is that, bilattice-based triangle is a more generalised algebraic structure that can be defined for any set of intervals over any complete lattice. However, preorder-based triangle can only be defined over subintervals of $[0,1]$ (or any interval of real numbers), because intervals in general complete lattices may not have midpoints or lengths.

\begin{definition}
A set of intervals in $L^I$ is said to be an m-set for a specific value $a \in [0,1]$ and is defined as:

\begin{center}
$m-set_a = \{ x \vert x \in L^I$ and $x_m = a\}$.
\end{center}

i.e. the set of intervals incomparable with the interval $[a,a]$ with respect to their degree of truth.
 \end{definition}

\section {Logical Operators on P(L)}
All the logical operators, e.g. conjunction, disjunction, implication and negation, defined for bilattice-based triangle (\textbf{B}(L)) \cite {cornelis2007uncertainty, deschrijver2007bilattice} are applicable for preorder-based triangle (\textbf{P}(L)) as well. But the modified truth and knowledge ordering will incorporate some modifications in the definition and properties of the connectives. The notations $0_{L^I}$ and $1_{L^I}$ stand for intervals $[0,0]$ and $[1,1]$ respectively.

\subsection{Negator:}

\begin{definition}
A negator on $(L^I, \leq_{t_p})$ is a decreasing mapping $\textbf{N}: L^I \longrightarrow L^I$, for which $\textbf{N}(0_{L^I}) = 1_{L^I}$ and $\textbf{N}(1_{L^I}) = 0_{L^I}$. If $\textbf{N}(\textbf{N}(x))=x$, then $\textbf{N}$ is involutive.
\end{definition}

\begin{theorem}
\label{negatorth}

Suppose there exists an involutive negator N on $([0,1],\leq)$. Then for all $x = [x_1,x_2]$ in $L^I$ the mapping $ \textbf{N}: L^I\longrightarrow L^I$ defined as
 \begin{center}
$\textbf{N}(x) = [N(x_2), N(x_1)]$
 \end{center}
is an involutive negator on $(L^I,\leq_{t_p})$

\end{theorem}

\begin{proof}
\textbf{N} to be an involutive negator it must satisfy the following criteria:
\vspace{0.1in}

\textbf{1.Boundary Condition:}

N being an involutive negator on $([0,1],\leq)$, $N(0)=1$ and $N(1)=0$.

Therefore,

\vspace{0.1in}
$\textbf{N}(0_{L^I})=\textbf{N}([0,0])=[N(0),N(0)] = [1,1]=1_{L^I}$.
\vspace{0.1in}

$\textbf{N}(1_{L^I})=\textbf{N}([1,1])=[N(1),N(1)] = [0,0]=0_{L^I}$.
\vspace{0.1in}

\textbf{2.}\textbf{N} has to be \textbf{decreasing} on $(L^I,\leq_{t_p})$.
\vspace{0.1in}

Let $x=[x_1,x_2]$ and $y=[y_1,y_2]$ are two intervals in $L^I$.

Now suppose, without loss of generality, $x \geq_{t_p} y$; which implies,

\begin{center}
$\frac{x_1+x_2}{2} \geq \frac{y_1+y_2}{2}$ or, $x_1+x_2 \geq y_1+y_2$.
\end{center}
Case 1: If neither of $x$ and $y$ is a sub-interval of the other, i.e.
\begin{center}
$x_1 \geq y_1$ and $x_2 \geq y_2$.
\end{center}

Hence, $N(x_1) \leq N(y_1)$ and $N(x_2) \leq N(y_2)$; since N is decreasing.

Therefore, $N(x_1) + N(x_2) \leq N(y_1) + N(y_2)$,

or, \  \  \  \  \  \  \  \  \ $\frac{N(x_1) + N(x_2)}{2} \leq \frac{N(y_1) + N(y_2)}{2}$,

or, \  \  \  \  \  \  \  \  \ $\textbf{N}(x) \leq_{t_p} \textbf{N}(y)$.

Hence, \textbf{N} is decreasing.
\vspace{0.1in}

Case 2: When $y$ is a sub-interval of $x$. Thus,
 
\begin{center}
$x_1 \leq y_1$ and $y_2 \leq x_2$
\end{center}
Hence, $N(x_1) \geq N(y_1)$ and $N(y_2) \geq N(x_2)$.

Since, $x \geq_{t_p} y$,  \ $x_1+x_2 \geq y_1+y_2$.

or, \  \  \  \  \  \  \  \  \  \  \  \  \  \  \ $x_2-y_2 \geq y_1 - x_1$.

Therefore, \  \  \  \  \ $N(y_2)-N(x_2) \geq N(x_1) - N(y_1)$; since N is decreasing.

or, \  \  \  \  \  \  \  \  \  \  \  \  \  \  \ $N(y_2)+N(y_1) \geq N(x_1) + N(x_2)$.

or, \  \  \  \  \  \  \  \  \  \  \  \  \  \  \ $\textbf{N}(y) \geq_{t_p} \textbf{N}(x)$.

Thus, \textbf{N} is decreasing.
\vspace{0.1in}

Case 3: When $x$ is a sub-interval of $y$. Then;

\begin{center}
$x_1 \geq y_1$ and $y_2 \geq x_2$
\end{center}
Hence, $N(x_1) \leq N(y_1)$ and $N(y_2) \leq N(x_2)$.

Since, $x \geq_{t_p} y$,  \ $x_1+x_2 \geq y_1+y_2$.

or, \  \  \  \  \  \  \  \  \  \  \  \  \  \  \ $x_1-y_1 \geq y_2 - x_2$.

Therefore, \  \  \  \  \ $N(y_1)-N(x_1) \geq N(x_2) - N(y_2)$; since N is decreasing.

or, \  \  \  \  \  \  \  \  \  \  \  \  \  \  \ $N(y_1)+N(y_2) \geq N(x_1) + N(x_2)$.

or, \  \  \  \  \  \  \  \  \  \  \  \  \  \  \ $\textbf{N}(y) \geq_{t_p} \textbf{N}(x)$.
\vspace{0.1in}

Thus, \textbf{N} is decreasing.
\vspace{0.15in}

Therefore, it is proved that \textbf{N} satisfies the boundary conditions and is a decreasing mapping on $(L^I, \leq_{t_p})$. So \textbf{N} is a negator on $(L^I, \leq_{t_p})$.
\vspace{0.1in}

Since, N is involutive, we obtain that, $\forall x \in [0,1]$;

$\textbf{N}(\textbf{N}(x))=\textbf{N}([N(x_2),N(x_1)])$

$=[N(N(x_1)),N(N(x_2))]=[x_1,x_2]=x$.
\vspace{0.1in}

Hence, \textbf{N} is involutive.

\end{proof}

\textbf{A standard negator} 

For an element $x=[x_1,x_2]$ in $L^I$ the standard negation of $x$ is defined as:
\begin{definition}
$\textbf{N}_s(x) = [1-x_2,1-x_1]$.
\end{definition}

Thus the degree of knowledge is unaltered by negation, but the interval (and hence its midpoint) is reflected across the central line of $L^I$ i.e. the line joining points $[0.5,0.5]$ and $[0,1]$. This negation corresponds to classical negation.
\vspace{0.15in}

\textbf{Properties:}
\vspace{0.1in}

1. $\textbf{N}_s(0_{L^I}) = 1_{L^I}$.
\vspace{0.1in}

2. $\textbf{N}_s$ is decreasing.
\vspace{0.1in}

3. $\textbf{N}_s$ is continuous.
\vspace{0.1in}

4. $\textbf{N}_s$ is involutive; i.e. $\textbf{N}_s(\textbf{N}_s(x))=x$.
\vspace{0.1in}

One point that must be emphasized is that involutive negators can be defined on $(L^I, \leq_{t_p})$ that are not of the form stated in Theorem \ref{negatorth}.
\vspace{0.1in}

\textbf{Example:} Consider the lattice $\textbf{L} = (\{0, 1/3, 2/3, 1\}, \leq)$ and a mapping $\textbf{N}_1$ on $(I(L), \leq_{t_p})$ defined as follows:
\vspace{0.1in}

\begin{center}
 $\textbf{N}_1([x_1, x_2]) = [1/3,2/3]$ if $[x_1,x_2]$ is $[0,1]$\\
 \  \  \  \  \  \  \  \  \  \  \  \  \  \  \ $=[0,1]$ if $[x_1,x_2]$ is $[1/3,2/3]$\\
 \  \  \  \  \  \  \  \  \  \  \  \  \  \  \ $= [1-x_2, 1-x_1]$ otherwise.\\

\end{center}

$\textbf{N}_1$ is an involutive negator on $(I(L), \leq_{t_p})$, but is not of the form specified in Theorem \ref{negatorth}. This is the difference between negators on bilattice-based triangle \cite {cornelis2007uncertainty} and preorder-based triangles.

\subsection{\textbf{T-norms and T-conorms:}}

The t-norms and t-conorms can be defined over the preorder-based triangle.

\begin{definition}
\label{defconj}

A conjunctor on $(L^I, \leq_{t_p})$ is an increasing $L^I \times {L^I \rightarrow L^I}$ mapping \textbf{T} satisfying \textbf{T}$(0_{L^I},0_{L^I})= $\textbf{T}$(0_{L^I},1_{L^I})$ = \textbf{T}$(1_{L^I},0_{L^I})=0_{L^I}$ and \textbf{T}$(1_{L^I},1_{L^I})=1_{L^I}$. 

A conjunctor is called a semi-norm if $(\forall x \in L^I)($\textbf{T}$(1_{L^I},x) = $\textbf{T}$(x,1_{L^I}) = x)$ and a semi-norm is called a t-norm if it is commutative and associative.
\end{definition}

\begin{definition}
\label{defdisj}
A disjunctor on $(L^I, \leq_{t_p})$ is an increasing $L^I \times {L^I \rightarrow L^I}$ mapping \textbf{S} satisfying \textbf{S}$(1_{L^I},0_{L^I})= $\textbf{S}$(0_{L^I},1_{L^I})$ = \textbf{S}$(1_{L^I},1_{L^I})=1_{L^I}$ and \textbf{S}$(0_{L^I},0_{L^I})=0_{L^I}$. 

A disjunctor is called a semi-conorm if $(\forall x \in L^I)($\textbf{S}$(0_{L^I},x) = $\textbf{S}$(x,0_{L^I}) = x)$ and a semi-conorm is called a t-conorm if it is commutative and associative.
\end{definition}

Two important class of t-(co)norms defined for IVFS, namely t-representable and pseudo t-representable t-(co)norms \cite {deschrijver2008additive}, can be defined over the preorder-based triangle structure.

\begin{definition}
A t-norm \textbf{T} on $(L^I, \leq_{t_p})$ is called t-representable if there exist t-norms $T_1$ and $T_2$ on $([0,1], \leq)$ such that $T_1 \leq T_2$ and such that \textbf{T} can be represented as, for all $x,y \in L^I$:

\begin{center}
\textbf{T}$(x,y) = [T_1(x_1,y_1), T_2(x_2,y_2)].$
\end{center} 
$T_1$ and $T_2$ are called representants of \textbf{T}.
\end{definition}

\begin{definition}
A t-norm \textbf{T} on $(L^I, \leq_{t_p})$ is called pseudo t-representable if there exists an t-norms \textbf{T} on $([0,1], \leq)$ such that for all $x,y \in I(L)$:

\begin{center}
\textbf{T}$(x,y) = [T(x_1,y_1), max (T(x_1,y_2), T(x_2,y_1)].$
\end{center} 
$T$ is called the representant of \textbf{T}.
\end{definition}

\subsubsection{Min t-norm and t-conorm}

The Min t-norm ($\textbf{T}_{Min}$) is the greatest t-norm with respect to the $\leq_{t}$ ordering and is defined as:
\vspace{0.1in}
\begin{center}
$\textbf{T}_{Min} = [min(x_1,y_1), min(x_2,y_2)]$.
\end{center}
One property of this t-norm is that it doesn't hold that $\forall x,y \in L^I$ either $\textbf{T}_{Min}(x,y) = x$ or $\textbf{T}_{Min}(x,y) = y$; for instance, $\textbf{T}_{Min}([0.1,0.5],[0.2,0.3]) = [0.1,0.3]$. This phenomenon is not intuitive sometimes. Using the modified truth ordering $(\leq_{t_p})$ a variant of the min t-norm can be defined over $(L^I, \leq_{t_p})$ as follows:
\vspace{0.1in}

\begin{definition}
For any two intervals $x,y \in L^I$ 
\begin{center}
$\textbf{T}_{Min_p}(x,y) = min_t \{x,y\}$ \ if  \ $ x_m \neq y_m$ \\
  \  \   \  \  \  \  \  \  \  \  \ \ = $max_k \{x,y\}$  \ if \  $x_m = y_m$
\end{center}
\end{definition}

\begin{definition}
For any two intervals $x,y \in L^I$
\begin{center}
$\textbf{S}_{Min_p}(x,y) = max_t \{x,y\}$ \ if  \ $ x_m \neq y_m$\\
  \  \   \  \  \  \  \  \  \  \  \ \ = $max_k \{x,y\}$  \ if \  $x_m = y_m$
\end{center}
\end{definition}

In the above definition $min_t\{x,y\}$  gives the interval having lower degree of truth irrespective of its knowledge content, i.e. $min_t\{x,y\} = x$ if $x \leq_{t_p} y$. Similar meaning can be ascribed to $max_t\{x,y\}$. Whereas,  $min_k\{x,y\}$ gives the interval which is lower with respect to the k-ordering, i.e. having higher degree of uncertainty. For instance, $min_k\{x,y\} = x$ if $ x\leq_{k_p} y$. Similarly $max_k\{x,y\}$ can be defined.

It is clear that $\textbf{T}_{Min_p}$ and $\textbf{S}_{Min_p}$ satisfies the conditions in Definition \ref{defconj} and \ref{defdisj} respectively.

\textbf{Example:} $\textbf{T}_{Min_p}([0.1,0.5],[0.2,0.3]) = [0.2,0.3]$. Thus, for all $x,y \in L^I$ either $\textbf{T}_{Min_p}(x,y) = x$ or $\textbf{T}_{Min_p}(x,y) = y$.

\begin{theorem}
The t-norm $\textbf{T}_{Min_p}$, t-conorm $\textbf{S}_{Min_p}$ and negator $\textbf{N}_s$ forms a De-Morgan triplet, i.e.\\
1. $\textbf{T}_{Min_p} (x,y) = \textbf{N}_s(\textbf{S}_{Min_p}(\textbf{N}_s(x),\textbf{N}_s(y)))$,\\
2. $\textbf{S}_{Min_p} (x,y) = \textbf{N}_s(\textbf{T}_{Min_p}(\textbf{N}_s(x),\textbf{N}_s(y)))$.
\end{theorem}

\begin{proof}
Consider two intervals $x,y \in L^I$.

\textbf{Part 1}:	First, suppose intervals $x$ and $y$ are comparable with respect to $\leq_{t_p}$, and lets assume, without loss of generality $x \geq_{t_p} y$. Thus $\textbf{T}_{Min_p}(x,y) = y$ and $\textbf{S}_{Min_p}(x,y) = x$. Since $\textbf{N}_s$ is decreasing with respect to the degree of truth, then $\textbf{N}_s(x) \leq_{t_p} \textbf{N}_s(y)$. So, from definition $\textbf{S}_{Min_p}(\textbf{N}_s(x),\textbf{N}_s(y)) = \textbf{N}_s(y)$. Thus $\textbf{N}_s(\textbf{S}_{Min_p}(\textbf{N}_s(x), \textbf{N}_s(y))) = y =\textbf{T}_{Min_p}(x,y)$.

Moreover, if $x_m = y_m$ (i.e. $x$ and $y$ are incomparable with respect to their degree of truth), and say, $x \leq_{k_p} y$ $\textbf{T}_{Min_p}(x,y) = y$. Since the negator $\textbf{N}_s$ preserves the degree of knowledge and reverses the degree of truth, $\textbf{N}_s(x)$ and $\textbf{N}_s(y)$ are incomparable in t-ordering and $\textbf{N}_s(x) \leq_{k_p} \textbf{N}_s(y)$. Thus, from definition $\textbf{S}_{Min_p}(\textbf{N}_s(x),\textbf{N}_s(y)) = \textbf{N}_s(y)$ and $\textbf{N}_s(\textbf{S}_{Min_p}(\textbf{N}_s(x),\textbf{N}_s(y)))= y =\textbf{T}_{Min_p}(x,y)$.
\vspace{0.1in}

\textbf{Part 2}: If $x \leq_{t_p} y$, then $\textbf{N}_s(x) \geq_{t_p} \textbf{N}_s(y)$ and $\textbf{T}_{Min_p}(\textbf{N}_s(x),\textbf{N}_s(y)) = \textbf{N}_s(y)$. Thus $\textbf{N}_s(\textbf{T}_{Min_p}(\textbf{N}_s(x), \textbf{N}_s(y))) = y =\textbf{S}_{Min_p}(x,y)$.

Moreover, if $x_m = y_m$ and say, $x \leq_{k_p} y$ $\textbf{S}_{Min_p}(x,y) = y$. The negator $\textbf{N}_s$ being order preserving for k-ordering, $\textbf{N}_s(x) \leq_{k_p} \textbf{N}_s(y)$. Thus, from definition $\textbf{T}_{Min_p}(\textbf{N}_s(x),\textbf{N}_s(y)) = \textbf{N}_s(y)$ and $\textbf{N}_s(\textbf{T}_{Min_p}(\textbf{N}_s(x),\textbf{N}_s(y)))= y =\textbf{S}_{Min_p}(x,y)$.
  
Hence, the t-norm $\textbf{T}_{Min_p}$, t-conorm $\textbf{S}_{Min_p}$ and negator $\textbf{N}_s$ forms a De-Morgan triplet.

\end{proof}

\subsubsection{Product t-norm and t-conorm}

The product t-(co)norm is useful to model the conjunction of independent events in probabilistic semantics. The t-representable and pseudo t-representable proudct t-(co)norms on $(L^I, \leq_{t_p})$ are defined in the same way as defined on $(L^I, \leq_{t})$. 

\begin{definition}
For any two intervals $x,y \in L^I$, the product t-norm is defined as follows:

\begin{center}
$\textbf{T}_{pr}([x_1,x_2],[y_1,y_2]) = [x_1y_1, x_2y_2]$, (t-representable)\\
$\textbf{T}_{ppr}([x_1,x_2],[y_1,y_2]) = [x_1y_1, max (x_1y_2, x_2y_1)]$, (pseudo t-representable)\\
\end{center}
\end{definition}

\begin{theorem}
For any $x,y \in L^I$
\begin{center}
$T_{pr} \geq_{t_p} T_{ppr}$
\end{center}
\end{theorem}

The proof of the above theorem is straightforward. 

\begin {definition}
The t-representable t-conorm can be defined as:
\begin{center}
$\textbf{S}_{pr}([x_1,x_2],[y_1,y_2]) = [1-(1-x_1) \times (1-y_1), 1-(1-x_2) \times (1-y_2)]$\\
\end{center}
\end{definition}

\begin{theorem}
The t-norm $\textbf{T}_{pr}$, t-conorm $\textbf{S}_{pr}$ and the standard negator $\textbf{N}_s$ forms a De-Morgan triplet, i.e.

1. $\textbf{T}_{pr} (x,y) = \textbf{N}_s (\textbf{S}_{pr}(\textbf{N}_s(x),\textbf{N}_s(y)))$,

2. $\textbf{S}_{pr} (x,y) = \textbf{N}_s (\textbf{T}_{pr}(\textbf{N}_s(x),\textbf{N}_s(y)))$,
\end{theorem}

\begin{proof}
Consider any two intervals $[x_1,x_2],[y_1,y_2] \in L^I$. 
 
1.$\textbf{S}_{pr}(\textbf{N}_s(x),\textbf{N}_s(y))$ 

$= \textbf{S}_{pr}([1-x_2,1-x_1],[1-y_2,1-y_1])$ 

$= [1 - x_2 \times y_2, 1 - x_1 \times y_1]$.

Now, $\textbf{N}_s(\textbf{S}_{pr}(\textbf{N}_s(x),\textbf{N}_s(y))) = \textbf{N}_s([1 - x_2 \times y_2, 1 - x_1 \times y_1])$
$= [x_1\times y_1, x_2 \times y_2] = \textbf{T}_{pr}([x_1,x_2],[y_1,y_2])$ 
\vspace{0.15in}

2. $\textbf{T}_{pr}(\textbf{N}_s(x),\textbf{N}_s(y))$
 
$= \textbf{T}_{pr}([1-x_2,1-x_1],[1-y_2,1-y_1])$
 
$= [(1 - x_2) \times (1 - y_2), (1 - x_1) \times (1 - y_1)]$.

Now, $\textbf{N}_s(\textbf{T}_{pr}(\textbf{N}_s(x),\textbf{N}_s(y)))$

$ = \textbf{N}_s((1 - x_2) \times (1 - y_2), (1 - x_1) \times (1 - y_1))$

$= [1 - (1 - x_1) \times (1 - y_1), 1 - (1 - x_2) \times (1 - y_2)]$

$ = \textbf{S}_{pr}([x_1,x_2],[y_1,y_2])$; (from definition).\\
The t-norm $\textbf{T}_{pr}$, t-conorm $\textbf{S}_{pr}$ and the standard negator $\textbf{N}_s$ forms a De-Morgan triplet.

\end{proof}

Thus, the preorder-based triangle structure offers us the flexibility to choose t-norms and t-conorms already defined for bilattice-based triangle or to define new connectives in accordance to the newly defined t-ordering and k-ordering.

\subsection{Implicators:}

\begin{definition} 
An implicator on $(L^I, \leq_{t_p})$ is a hybrid monotonous $L^I \times L^I \rightarrow L^I$ mapping \textbf{I}(i.e. a mapping with decreasing first and increasing second partial mapping) that satisfies \textbf{I}($0_{L^I},0_{L^I}$) = \textbf{I}($0_{L^I},1_{L^I}$) = \textbf{I}($1_{L^I},1_{L^I}$) $=1_{L^I}$ and \textbf{I}($1_{L^I},0_{L^I}$) $=0_{L^I}$.
\end{definition}

One of the common class of implicators are Strong-implicators or S-implicators in short.

\begin{definition}
For two intervals $x,y \in L^I$ and any t-conorm \textbf{S} and negator \textbf{N} on $(L^I, \leq_t)$ the S-implicator generated by \textbf{S} and \textbf{N} is 
\begin{center}
$\textbf{I}_{S,N}(x,y) = \textbf{S}(\textbf{N}(x),y)$.
\end{center}
\end{definition}

The S-implicators defined for the structure $(L^I,\leq_{t_p})$ are similar to those defined for $(L^I,\leq_t)$, and are not discussed further.

There is another important class of implicators, namely R-implicators, generated as residuum of some t-norms on $(L^I,\leq_t)$. An R-implicator on $(L^I,\leq_t)$ generated by a t-norm \textbf{T} is defined as:
\begin{center}
$\textbf{I}_R(x,y) = Sup \{\gamma \in L^I| \textbf{T}(x,\gamma) \leq_t y\}$.
\end {center}
Now, because the definition involves truth ordering, the modified definition of $\leq_{t_p}$ demands modification to the definition of R-implicator.

\begin{definition}
For a t-norm \textbf{T} defined on $(L^I, \leq_{t_p})$ an R-implicator generated from \textbf{T} is defined as:

$\textbf{I}_{R_{t_p}}(x,y) = Sup_{t_p}\{ \gamma \in L^I| \textbf{T}(x,\gamma) \leq_{t_p} y$ or $[\textbf{T}(x, \gamma)]_m = y_m \}$

\end{definition}
where, $Sup_{t_p}$ is the interval having maximum degree of truth. Sometimes, instead of a unique value, the operation $Sup_{t_p}$ may give a set of intervals belonging to the same m-set and hence $\textbf{I}_{R_{t_p}}(x,y)$ may not be unique.

\textbf{Example:} Suppose $\textbf{L} = ([0,1], \leq)$ and the t-norm is $\textbf{T}_{Min_p}$. Then the R-implicator generated from this t-norm is given by:
\begin{center}
$\textbf{I}_{Min} = Sup_{t_p} \{\gamma \in L^I| \textbf{T}_{Min_p} (x, \gamma) \leq_{t_p} y$ or $[\textbf{T}_{Min_p} (x,\gamma)]_m = y_m\}$.
\end{center}

1. If $x \leq_{t_p} y$, for any $\gamma \in L^I$, $\textbf{T}_{Min_p} (x, \gamma) \leq_{t_p} y$. Thus, $\textbf{I}_{Min} = [1,1]$.

2. If $x_m = y_m$, then for any $\gamma \geq_{t_p} x$, $[\textbf{T}_{Min_p} (x, \gamma)]_m = y_m$. Thus, $\textbf{I}_{Min} = [1,1]$.

3. If $x \geq_{t_p} y$, then for any interval $\gamma$ with $\gamma_m = y_m$, we have $[\textbf{T}_{Min_p} (x, \gamma)]_m = y_m$. Thus, $\textbf{I}_{Min} = \gamma$ s.t. $\gamma_m = y_m$. Hence, the implicator does not give a unique element, but an $m-set_a$ of intervals with $a = y_m$.

\textbf{Example:} Suppose $\textbf{L} = ([0,1], \leq)$ and the t-norm is $\textbf{T}_{pr}$. 

Then the R-implicator generated from this t-norm is given by:
\begin{center}
$\textbf{I}_{pr} = Sup_{t_p} \{\gamma \in L^I| \textbf{T}_{pr} (x, \gamma) \leq_{t_p} y$ or $[\textbf{T}_{pr} (x,\gamma)]_m = y_m\}$.
\end{center}
or, in other words,
\begin{center}
$\textbf{I}_{pr} = Sup_{t_p} \{\gamma \in L^I|(x_1 \times \gamma_1 + x_2 \times \gamma_2$  $\leq y_1 + y_2)$ or $(x_1 \times \gamma_1 = y_1$ and $x_2 \times \gamma_2 = y_2) \}$.
\end{center}

Case 1: If $x_1 + x_2 \leq y_1 + y_2$; $\textbf{I}_{pr} = [1,1]$.

Case 2: When $x_1 + x_2 > y_1+ y_2$, i.e. $x >_{t_p} y$ and no interval resides completely in the other, i.e. $y_1 \leq x_1$ and $y_2 \leq x_2$;

$\textbf{I}_{pr} = max_{t_p} ([\frac{y_1}{x_1}, \frac{y_2}{x_2}] , [\frac{y_1 + y_2}{x_1 + x_2}, \frac{y_1 + y_2}{x_1 + x_2}])$.

Note: $[\frac{y_1}{x_1}, \frac{y_2}{x_2}]_m - [\frac{y_1 + y_2}{x_1 + x_2}, \frac{y_1 + y_2}{x_1 + x_2}]_m = \frac{(x_2 - x_1)(y_1x_2 - y_2x_1)}{2x_1x_2(x_1 + x_2)}$.\\
Thus, 
\begin{center}
$\textbf{I}_{pr} = [\frac{y_1}{x_1}, \frac{y_2}{x_2}]$ if $\frac{y_1}{y_2} > \frac{x_1}{x_2}$,\\
 \  \  \  \  \  \  \   \  \  \  \ $= [\frac{y_1 + y_2}{x_1 + x_2}, \frac{y_1 + y_2}{x_1 + x_2}]$ otherwise.\\
\end{center}

Case 3: When $x_1 + x_2 > y_1+ y_2$, i.e. $x >_{t_p} y$ and one interval resides completely in the other, i.e. either $x_1 \leq y_1 \leq y_2 < x_2$ or $y_1 < x_1 \leq x_2 \leq y_2$; $\textbf{I}_{pr} = [\gamma, \gamma]$ where, $\gamma = \frac{y_1 + y_2}{x_1 + x_2}$, since, $[x_1 \times \gamma_1, x_2 \times \gamma_2]_m = \gamma \times \frac{x_1 + x_2}{2} = y_m$.

\section{Modeling reasoning problems using the Preorder-based triangle:}

This section is devoted to demonstrate how the modified algebraic structure, namely the preorder-based triangle, can be employed to model the motivating examples shown in section 3. It is also demonstrated that how the preorder-based triangle can solve the reasoning problems mentioned before.

1. In Example 1, which presented an intuitive explanation regarding the inadequacy of knowledge ordering, initially the experts assigned the epistemic state $[0.5,1]$ to the statement "Tweety Flies" based on incomplete knowledge. Then when they came to know that Tweety is a penguin they reassessed the epistemic state of "Tweety Flies" as $[0,0]$. However, since intervals $[0.5,1]$ and $[0,0]$ are incomparable with respect to $\leq_k$ in bilattice-based triangle it could not be decided which one is a surer assertion and which one should be taken as the final assessment. In preorder-based triangle this dilemma can be solved since in the newly defined knowledge ordering we can see $[0.5,1] \leq_{k_p} [0,0]$ (since $(1-0.5) > (0-0)$) and hence it can be seen that interval $[0,0]$ is placed higher in the modified k-ordering than $[0.5,1]$. Thus the new ordering prompts to choose the definite fact "Tweety doesn't fly" (with the assigned interval [0,0]) over the default fact "Tweety flies" (with the assigned interval [0.5,1]).

2. Example 2 dealt with logical reasoning in a visual surveillance system for human detection and based on the given information we tried to reason about whether two individuals '$a$' and '$b$' were same person or not. Performing the reasoning it was found that:

\begin{center}

$cl_+(\phi)(equal(a,b)) = [0.5,0.8]$

$cl_-(\phi)(equal(a,b)) = [0,0.1]$

\end{center}

where, $cl_+(\phi)(equal(a,b))$ $(cl_-(\phi)(equal(a,b)) )$accounts for the belief in support of (against to) the fact that $a$ and $b$ are same person. Now to reach a final conclusion we must consider which of the positive and negative evidence provides with a stronger and surer belief. Hence the two epistemic states have to be compared with respect to the knowledge ordering and we must go with the one placed higher in knowledge ordering. But it was seen that with the knowledge ordering in the bilattice-based triangle the two intervals were incomparable. Though it can be seen that interval $[0.5,0.8]$ is wider interval than that of $[0,0.1]$ and hence intuitively the later one is a more certain assessment.

In the preorder-based triangle $[0.5,0.8] \leq_{k_p} [0,0.1]$ and thus the interval $lub_k ([0.5,0.8],[0,0.1]) = [0,0.1]$, will be taken as the final assertion of $equal(a,b)$. 

This result can be achieved by an additional step in the reasoning:

\begin{center}
$\phi[equal(a,b)] = lub_k (cl_+(\phi)(equal(a,b)), cl_-(\phi)(equal(a,b)))$.
\end{center}

There is still a chance that, for two intervals $lub_k$ doesn't exist as the triangle structure is not complete with knowledge ordering. In that case the indecision is justified since both the assertions give same amount of information, i.e. the corresponding intervals are of equal length. Hence in such a situation human intervention is necessary or some other parameter can be used to reach to a conclusion depending on the application.

But it is to be emphasized that the modified knowledge ordering in the preorder based triangle gives intuitive results in situations where bilattice-based triangle fails to do so and hence the former has a wider applicability in real life common sense reasoning.

3. In Example 3, reasoning in an artificial triage system was investigated. It was assumed that a patient is suffering from two diseases($di1$, $di2$), whose drugs are mutually incompatible and the sequence of treatment could not be decided as the diseases have severity levels specified by intervals $[0.4,0.9]$ and $[0.5,0.6]$ which are incomparable with respect to t-ordering of bilattice-based triangle. The intervals being incomparable by truth ordering, rules 3 and 4, as specified in the example, become unusable. 

However the modified truth ordering in preorder-based triangle can resolve this conflict. Firstly the rules 3 and 4 in example 3 are modified as follows:

3'. $\phi [(\phi[di2] \leq_{t_p} \phi[di1]) \longrightarrow dr1] = [1,1]$

4'. $\phi [(\phi[di1] \leq_{t_p} \phi[di2]) \longrightarrow dr2] = [1,1]$

Now, in the current scenario Rule 3' is invoked, inferring to administer drug $dr1$, since we have $[0.5,0.6] \leq_{t_p} [0.4,0.9]$. This is intuitive; since we have $[0.5,0.6]_m = 0.55 < 0.65 = [0.4,0.9]_m$ and hence from Theorem \ref{theorem:midpoint}, $Prob (\hat{i}_{d1} \leq \hat{i}_{d2}) \leq Prob(\hat{i}_{d2} \leq \hat{i}_{d1})$, i.e. though the severity degrees of the two diseases have uncertainties but it is more probable that $di1$ has higher severity than that of $di2$. This probabilistic intuition was not reflected in the truth ordering of bilattice-based triangle.

Here also some situations may arise where intervals are not comparable with respect to $\leq_{t_p}$, as is depicted by Figure \ref{fig:incomparable}. But that is also justified, since in that case, the probability that disease $di1$ is more severe is equal to the probability that disease $di2$ is more severe and nothing can be decided. In these cases, depending on the application, some other parameter have to be chosen to break the tie. Thus the modified truth ordering in the preorder-based triangle is more intuitive and suitable for practical applications.

Therefore as a whole the aforementioned examples demonstrate that, the inference capacity and applicability of $P(L)$-based reasoning systems are broader than the bilattice-based systems. 

\vspace{0.15in}
\textbf{Relationship of $P(L)$ with existing well-known nonmonotonic reasoning formalism:}

Answer Set Programming (ASP) is a well established nonmonotonic reasoning paradigm for commonsense reasoning in discrete domain. This reasoning formalism cannot deal with vague and uncertain information. Possibilistic Fuzzy Answer Set Programming (PFASP) \cite{bauters2010towards} is an extension of ASP for reasoning in continuous domain in presence of uncertainties. Thus PFASP performs nonmonontonic reasoning in presence of both vagueness and uncertainty. 

In PFASP two real numbers taken from the unit interval $[0,1]$ are used to specify the degree of vagueness and degree of possibility. The assignment of the two numbers are actually independent of each other. However we can argue that  assigning two independent numbers as the degree of vagueness and the degree of possibility of a piece of information is not always intuitive. Because in human commonsense reasoning our assessment of the degree of truth of a proposition is somewhat dependent on the certainty about that proposition.

In PFASP, there is nothing to prevent us from using a fact as:

\begin{center}
$0.5: a \longleftarrow 1$,
\end{center}

Here, the possibility degree $0.5$ signifies that the truth of $a$, i.e. $a$ getting truth value 1, is neither possible nor impossible, i.e. nothing can be said about whether $a$ is true or not. In other words we have no knowledge at all about $a$. Thus, the facts $0.5: a \longleftarrow 0$, $0.5: a \longleftarrow 1$, $0.5: a \longleftarrow 0.5$ are all the same as the convey no information about $a$'s truth degree. Also, it is unjustified to assert any truth value to $a$ if we possess no knowledge about the statement. Moreover the facts, $0.5: a \longleftarrow 0.6$ and $0.5: b \longleftarrow 0.8$ prompt to infer that $b$ is more true than $a$. But the possibility degree of 0.5 for both the facts suggests that we have no knowledge about $a$ or $b$; hence comparing the truth of $a$ and $b$ is not even meaningful.

In this situation ascribing the interval $[0,1]$ to such a statement would be more compact, straightforward and intuitive. Moreover, a range of values comprising an interval is a more natural representation of uncertainty or lack of unanimity of multiple experts.
 
Therefore the preorder-based triangle can be employed as the truth value space or set of epistemic states in Answer Set programming to give rise to a framework capable of performing nonmonotonic reasoning with vague and uncertain information.

To appreciate the reasoning power of the preorder-based triangle structure we can consider an example stated in \cite{bauters2010towards} and try to reformulate it using the proposed $P(L)$. The example involves reasoning about risk of roads during snow and low temperature. The rules specified in the example are as follows:

r1: $1: cold \longleftarrow 0.6$

r2: $1: wet \longleftarrow 0.4$

r3: $1: risky \longleftarrow cold. snow$

r4: $0.8: snow \longleftarrow (cold \geq 0.5) \wedge wet$

r5: $0.6: risky \longleftarrow 0.5. cold$

r6: $0.6: risky \longleftarrow 0.8. wet$

Performing the reasoning, using consequence operators, the final belief set comes to be,

$S_1 = \{cold^{1;0.6}, wet^{1;0.4}, snow^{0.8;0.4}, risky^{0.6,0.32}, risky^{0.8,0.24}\}$

Here, a candidate $x^{c;t}$ denotes truth value of $x$ can be asserted to be $t$ with certainty degree $c$.

Remodeling using ASP based on preorder-based triangle the above rules can be transformed into the following:

r'1: $cold \longleftarrow [0.6,0.6]$

r'2: $wet \longleftarrow [0.4, 0.4]$

r'3: $risky \stackrel{[1,1]}{\leftarrow} cold. snow$

r'4: $snow \stackrel{[0.8,1]}{\leftarrow} T_{pr}((cold \geq_t [0.5,0.5]), wet)$

r'5: $risky \stackrel{[0.3,0.7]}{\leftarrow} cold$

r'6: $risky \stackrel{[0.6,1]}{\leftarrow} wet$

The weight of the rules are subintervals of $[0,1]$ and thus denotes the level of vagueness and uncertainty level regarding the rule. In other words, the weight denotes what would be the epistemic state of the consequent of the rule when the antecedent is absolutely true, i.e. assigned with $[1,1]$. It can be noted that rules r1 and r2 are actually facts representing certain fuzzy information and hence in the transformed program rules r'1 and r'2 have exact intervals as their weights. Rules r'5 and r'6 have wider intervals as their weights than rule r'4. This signifies that rule r'4 is a more certain rule than r'5 or r'6. If the antecedent of a rule gets an epistemic state other than $[1,1]$, say $a$, then the epistemic state of the consequent would be $T_{pr}(a,w)$, where $w$ is the weight of the rule. If a particular atom is consequent of more than one rules then the final assigned epistemic state can be obtained by taking the least upper bound of the individual epistemic states.

Now, performing similar reasoning with the transformed rules the resulting belief set becomes:

$S_2 = \{cold: [0.6,0.6], wet: [0.4,0.4], snow: [0.32,0.4],\\ risky: lub_t([0.24,0.24], [0.18,0.42], [0.24,0.4])\}$

or, $S_2 = \{cold: [0.6,0.6], wet: [0.4,0.4], snow: [0.32,0.4], risky: [0.24,0.4]\}$.

It can be seen that the set $S_2$ provides with range of possible values for each atomic statement. This representation is easier to comprehend. Therefore $P(L)$ is suitable to be used to develop nonmonotonic reasoning formalisms to deal with fuzzy and uncertain information.

\section {Conclusion:}

We conclude with a critical appreciation of the proposed structure with respect to the bilattice-based triangle. The structure, preorder-based triangle, together with the logical operators defined on it, provides a framework for reasoning with imprecise, uncertain and incomplete information. Unlike bilattice-based triangle, the preorder-based triangle is capable of handling repetitive belief revisions in nonmonotonic reasoning. Moreover the truth ordering in the new structure is more intuitive. The fact that the knowledge and truth order now become fully orthogonal does have a strong appeal in application areas involving nonmonotonic logical reasoning with vague and incomplete information. As demonstrated here, all the operators defined for bilattice-based triangles are suitable for the proposed structure as well and the modified truth ordering invokes some new logical connectives with interesting properties. Thus, the proposed preorder-based structure can be considered as an enhancement to bilattice-based triangle.  

This work is an preliminary analysis of the necessity of preorder-based triangle and its pros and cons, and leaves enough scope for further investigation and analysis.

\bibliographystyle{apacite}
\renewcommand\bibliographytypesize{\fontsize{10}{12}\selectfont}
\bibliography{biblist}

\begin{thebibliography}{}

\bibitem [\protect \citeauthoryear {%
Arieli%
, Cornelis%
, Deschrijver%
\BCBL {}\ \BBA {} Kerre%
}{%
Arieli%
\ \protect \BOthers {.}}{%
{\protect \APACyear {2004}}%
}]{%
arieli2004relating}
\APACinsertmetastar {%
arieli2004relating}%
\begin{APACrefauthors}%
Arieli, O.%
, Cornelis, C.%
, Deschrijver, G.%
\BCBL {}\ \BBA {} Kerre, E.%
\end{APACrefauthors}%
\unskip\
\newblock
\APACrefYearMonthDay{2004}{}{}.
\newblock
{\BBOQ}\APACrefatitle {Relating intuitionistic fuzzy sets and interval-valued
  fuzzy sets through bilattices} {Relating intuitionistic fuzzy sets and
  interval-valued fuzzy sets through bilattices}.{\BBCQ}
\newblock
\APACjournalVolNumPages{Applied Computational Intelligence}{}{}{57--64}.
\PrintBackRefs{\CurrentBib}

\bibitem [\protect \citeauthoryear {%
Arieli%
, Cornelis%
, Deschrijver%
\BCBL {}\ \BBA {} Kerre%
}{%
Arieli%
\ \protect \BOthers {.}}{%
{\protect \APACyear {2005}}%
}]{%
arieli2005bilattice}
\APACinsertmetastar {%
arieli2005bilattice}%
\begin{APACrefauthors}%
Arieli, O.%
, Cornelis, C.%
, Deschrijver, G.%
\BCBL {}\ \BBA {} Kerre, E.%
\end{APACrefauthors}%
\unskip\
\newblock
\APACrefYearMonthDay{2005}{}{}.
\newblock
{\BBOQ}\APACrefatitle {Bilattice-based squares and triangles} {Bilattice-based
  squares and triangles}.{\BBCQ}
\newblock
\BIn{} \APACrefbtitle {European Conference on Symbolic and Quantitative
  Approaches to Reasoning and Uncertainty} {European conference on symbolic and
  quantitative approaches to reasoning and uncertainty}\ (\BPGS\ 563--575).
\PrintBackRefs{\CurrentBib}

\bibitem [\protect \citeauthoryear {%
Bauters%
, Schockaert%
, De~Cock%
\BCBL {}\ \BBA {} Vermeir%
}{%
Bauters%
\ \protect \BOthers {.}}{%
{\protect \APACyear {2014}}%
}]{%
bauters2014semantics}
\APACinsertmetastar {%
bauters2014semantics}%
\begin{APACrefauthors}%
Bauters, K.%
, Schockaert, S.%
, De~Cock, M.%
\BCBL {}\ \BBA {} Vermeir, D.%
\end{APACrefauthors}%
\unskip\
\newblock
\APACrefYearMonthDay{2014}{}{}.
\newblock
{\BBOQ}\APACrefatitle {Semantics for possibilistic answer set programs:
  uncertain rules versus rules with uncertain conclusions} {Semantics for
  possibilistic answer set programs: uncertain rules versus rules with
  uncertain conclusions}.{\BBCQ}
\newblock
\APACjournalVolNumPages{International Journal of Approximate
  Reasoning}{55}{2}{739--761}.
\PrintBackRefs{\CurrentBib}

\bibitem [\protect \citeauthoryear {%
Bauters%
, Schockaert%
, Janssen%
, Vermeir%
\BCBL {}\ \BBA {} De~Cock%
}{%
Bauters%
\ \protect \BOthers {.}}{%
{\protect \APACyear {2010}}%
}]{%
bauters2010towards}
\APACinsertmetastar {%
bauters2010towards}%
\begin{APACrefauthors}%
Bauters, K.%
, Schockaert, S.%
, Janssen, J.%
, Vermeir, D.%
\BCBL {}\ \BBA {} De~Cock, M.%
\end{APACrefauthors}%
\unskip\
\newblock
\APACrefYearMonthDay{2010}{}{}.
\newblock
{\BBOQ}\APACrefatitle {Towards possibilistic fuzzy answer set programming}
  {Towards possibilistic fuzzy answer set programming}.{\BBCQ}
\newblock
\BIn{} \APACrefbtitle {13th Non-Monotonic Reasoning Workshop (NMR 2010);
  collocated with 12th International conference on the Principles of Knowledge
  Representation and Reasoning (KR 2010).} {13th non-monotonic reasoning
  workshop (nmr 2010); collocated with 12th international conference on the
  principles of knowledge representation and reasoning (kr 2010).}
\PrintBackRefs{\CurrentBib}

\bibitem [\protect \citeauthoryear {%
Brewka%
}{%
Brewka%
}{%
{\protect \APACyear {1991}}%
}]{%
brewka1991nonmonotonic}
\APACinsertmetastar {%
brewka1991nonmonotonic}%
\begin{APACrefauthors}%
Brewka, G.%
\end{APACrefauthors}%
\unskip\
\newblock
\APACrefYear{1991}.
\newblock
\APACrefbtitle {Nonmonotonic reasoning: logical foundations of commonsense}
  {Nonmonotonic reasoning: logical foundations of commonsense}\ (\BVOL~12).
\newblock
\APACaddressPublisher{}{Cambridge University Press}.
\PrintBackRefs{\CurrentBib}

\bibitem [\protect \citeauthoryear {%
Burke%
\ \BBA {} Madison%
}{%
Burke%
\ \BBA {} Madison%
}{%
{\protect \APACyear {1990}}%
}]{%
burke1990artificial}
\APACinsertmetastar {%
burke1990artificial}%
\begin{APACrefauthors}%
Burke, M\BPBI D.%
\BCBT {}\ \BBA {} Madison, D\BPBI E.%
\end{APACrefauthors}%
\unskip\
\newblock
\APACrefYearMonthDay{1990}{}{}.
\newblock
{\BBOQ}\APACrefatitle {Artificial intelligence in emergency department triage.}
  {Artificial intelligence in emergency department triage.}{\BBCQ}
\newblock
\APACjournalVolNumPages{The Journal of ambulatory care
  management}{13}{3}{50--54}.
\PrintBackRefs{\CurrentBib}

\bibitem [\protect \citeauthoryear {%
Cornelis%
, Arieli%
, Deschrijver%
\BCBL {}\ \BBA {} Kerre%
}{%
Cornelis%
\ \protect \BOthers {.}}{%
{\protect \APACyear {2007}}%
}]{%
cornelis2007uncertainty}
\APACinsertmetastar {%
cornelis2007uncertainty}%
\begin{APACrefauthors}%
Cornelis, C.%
, Arieli, O.%
, Deschrijver, G.%
\BCBL {}\ \BBA {} Kerre, E\BPBI E.%
\end{APACrefauthors}%
\unskip\
\newblock
\APACrefYearMonthDay{2007}{}{}.
\newblock
{\BBOQ}\APACrefatitle {Uncertainty modeling by bilattice-based squares and
  triangles} {Uncertainty modeling by bilattice-based squares and
  triangles}.{\BBCQ}
\newblock
\APACjournalVolNumPages{IEEE Transactions on fuzzy Systems}{15}{2}{161--175}.
\PrintBackRefs{\CurrentBib}

\bibitem [\protect \citeauthoryear {%
Deschrijver%
}{%
Deschrijver%
}{%
{\protect \APACyear {2008}}%
}]{%
deschrijver2008additive}
\APACinsertmetastar {%
deschrijver2008additive}%
\begin{APACrefauthors}%
Deschrijver, G.%
\end{APACrefauthors}%
\unskip\
\newblock
\APACrefYearMonthDay{2008}{}{}.
\newblock
{\BBOQ}\APACrefatitle {Additive generators in interval-valued fuzzy set theory}
  {Additive generators in interval-valued fuzzy set theory}.{\BBCQ}
\newblock
\BIn{} \APACrefbtitle {Proceedings of IPMU} {Proceedings of ipmu}\ (\BVOL~8,
  \BPG~1337).
\PrintBackRefs{\CurrentBib}

\bibitem [\protect \citeauthoryear {%
Deschrijver%
}{%
Deschrijver%
}{%
{\protect \APACyear {2009}}%
}]{%
deschrijver2009generalized}
\APACinsertmetastar {%
deschrijver2009generalized}%
\begin{APACrefauthors}%
Deschrijver, G.%
\end{APACrefauthors}%
\unskip\
\newblock
\APACrefYearMonthDay{2009}{}{}.
\newblock
{\BBOQ}\APACrefatitle {Generalized arithmetic operators and their relationship
  to t-norms in interval-valued fuzzy set theory} {Generalized arithmetic
  operators and their relationship to t-norms in interval-valued fuzzy set
  theory}.{\BBCQ}
\newblock
\APACjournalVolNumPages{Fuzzy Sets and Systems}{160}{21}{3080--3102}.
\PrintBackRefs{\CurrentBib}

\bibitem [\protect \citeauthoryear {%
Deschrijver%
, Arieli%
, Cornelis%
\BCBL {}\ \BBA {} Kerre%
}{%
Deschrijver%
\ \protect \BOthers {.}}{%
{\protect \APACyear {2007}}%
}]{%
deschrijver2007bilattice}
\APACinsertmetastar {%
deschrijver2007bilattice}%
\begin{APACrefauthors}%
Deschrijver, G.%
, Arieli, O.%
, Cornelis, C.%
\BCBL {}\ \BBA {} Kerre, E\BPBI E.%
\end{APACrefauthors}%
\unskip\
\newblock
\APACrefYearMonthDay{2007}{}{}.
\newblock
{\BBOQ}\APACrefatitle {A bilattice-based framework for handling graded truth
  and imprecision} {A bilattice-based framework for handling graded truth and
  imprecision}.{\BBCQ}
\newblock
\APACjournalVolNumPages{International Journal of Uncertainty, Fuzziness and
  Knowledge-Based Systems}{15}{01}{13--41}.
\PrintBackRefs{\CurrentBib}

\bibitem [\protect \citeauthoryear {%
Dubois%
}{%
Dubois%
}{%
{\protect \APACyear {2008}}%
}]{%
dubois2008ignorance}
\APACinsertmetastar {%
dubois2008ignorance}%
\begin{APACrefauthors}%
Dubois, D.%
\end{APACrefauthors}%
\unskip\
\newblock
\APACrefYearMonthDay{2008}{}{}.
\newblock
{\BBOQ}\APACrefatitle {On ignorance and contradiction considered as
  truth-values} {On ignorance and contradiction considered as
  truth-values}.{\BBCQ}
\newblock
\APACjournalVolNumPages{Logic Journal of IGPL}{16}{2}{195--216}.
\PrintBackRefs{\CurrentBib}

\bibitem [\protect \citeauthoryear {%
Esteva%
, Garcia-Calv{\'e}s%
\BCBL {}\ \BBA {} Godo%
}{%
Esteva%
\ \protect \BOthers {.}}{%
{\protect \APACyear {1994}}%
}]{%
esteva1994enriched}
\APACinsertmetastar {%
esteva1994enriched}%
\begin{APACrefauthors}%
Esteva, F.%
, Garcia-Calv{\'e}s, P.%
\BCBL {}\ \BBA {} Godo, L.%
\end{APACrefauthors}%
\unskip\
\newblock
\APACrefYearMonthDay{1994}{}{}.
\newblock
{\BBOQ}\APACrefatitle {Enriched interval bilattices and partial many-valued
  logics: an approach to deal with graded truth and imprecision} {Enriched
  interval bilattices and partial many-valued logics: an approach to deal with
  graded truth and imprecision}.{\BBCQ}
\newblock
\APACjournalVolNumPages{International Journal of Uncertainty, Fuzziness and
  Knowledge-Based Systems}{2}{01}{37--54}.
\PrintBackRefs{\CurrentBib}

\bibitem [\protect \citeauthoryear {%
Ginsberg%
}{%
Ginsberg%
}{%
{\protect \APACyear {1988}}%
}]{%
ginsberg1988multivalued}
\APACinsertmetastar {%
ginsberg1988multivalued}%
\begin{APACrefauthors}%
Ginsberg, M\BPBI L.%
\end{APACrefauthors}%
\unskip\
\newblock
\APACrefYearMonthDay{1988}{}{}.
\newblock
{\BBOQ}\APACrefatitle {Multivalued logics: A uniform approach to reasoning in
  artificial intelligence} {Multivalued logics: A uniform approach to reasoning
  in artificial intelligence}.{\BBCQ}
\newblock
\APACjournalVolNumPages{Computational intelligence}{4}{3}{265--316}.
\PrintBackRefs{\CurrentBib}

\bibitem [\protect \citeauthoryear {%
Goguen%
}{%
Goguen%
}{%
{\protect \APACyear {1967}}%
}]{%
goguen1967fuzzy}
\APACinsertmetastar {%
goguen1967fuzzy}%
\begin{APACrefauthors}%
Goguen, J\BPBI A.%
\end{APACrefauthors}%
\unskip\
\newblock
\APACrefYearMonthDay{1967}{}{}.
\newblock
{\BBOQ}\APACrefatitle {L-fuzzy sets} {L-fuzzy sets}.{\BBCQ}
\newblock
\APACjournalVolNumPages{Journal of mathematical analysis and
  applications}{18}{1}{145--174}.
\PrintBackRefs{\CurrentBib}

\bibitem [\protect \citeauthoryear {%
Golding%
, Wilson%
\BCBL {}\ \BBA {} Marwala%
}{%
Golding%
\ \protect \BOthers {.}}{%
{\protect \APACyear {2008}}%
}]{%
golding2008emergency}
\APACinsertmetastar {%
golding2008emergency}%
\begin{APACrefauthors}%
Golding, D.%
, Wilson, L.%
\BCBL {}\ \BBA {} Marwala, T.%
\end{APACrefauthors}%
\unskip\
\newblock
\APACrefYearMonthDay{2008}{}{}.
\newblock
{\BBOQ}\APACrefatitle {Emergency Centre Organization and Automated Triage
  System} {Emergency centre organization and automated triage system}.{\BBCQ}
\newblock
\APACjournalVolNumPages{arXiv preprint arXiv:0810.3671}{}{}{}.
\PrintBackRefs{\CurrentBib}

\bibitem [\protect \citeauthoryear {%
Nguyen%
, Kreinovich%
\BCBL {}\ \BBA {} Zuo%
}{%
Nguyen%
\ \protect \BOthers {.}}{%
{\protect \APACyear {1997}}%
}]{%
nguyen1997interval}
\APACinsertmetastar {%
nguyen1997interval}%
\begin{APACrefauthors}%
Nguyen, H\BPBI T.%
, Kreinovich, V.%
\BCBL {}\ \BBA {} Zuo, Q.%
\end{APACrefauthors}%
\unskip\
\newblock
\APACrefYearMonthDay{1997}{}{}.
\newblock
{\BBOQ}\APACrefatitle {Interval-valued degrees of belief: applications of
  interval computations to expert systems and intelligent control}
  {Interval-valued degrees of belief: applications of interval computations to
  expert systems and intelligent control}.{\BBCQ}
\newblock
\APACjournalVolNumPages{International Journal of Uncertainty, Fuzziness and
  Knowledge-Based Systems}{5}{03}{317--358}.
\PrintBackRefs{\CurrentBib}

\bibitem [\protect \citeauthoryear {%
Sambuc%
}{%
Sambuc%
}{%
{\protect \APACyear {1975}}%
}]{%
sambuc1975functions}
\APACinsertmetastar {%
sambuc1975functions}%
\begin{APACrefauthors}%
Sambuc, R.%
\end{APACrefauthors}%
\unskip\
\newblock
\APACrefYear{1975}.
\unskip\
\newblock
\APACrefbtitle {Functions\^{} flous: Application de l'Aide a Diagnostique en
  Pathologie Thyro{\"\i}dienne} {Functions\^{} flous: Application de l'aide a
  diagnostique en pathologie thyro{\"\i}dienne}\ \APACtypeAddressSchool
  {\BUPhD}{}{}.
\unskip\
\newblock
\APACaddressSchool {}{These Univ. de Marseille, Marseille}.
\PrintBackRefs{\CurrentBib}

\bibitem [\protect \citeauthoryear {%
Sandewall%
}{%
Sandewall%
}{%
{\protect \APACyear {1989}}%
}]{%
sandewall1989semantics}
\APACinsertmetastar {%
sandewall1989semantics}%
\begin{APACrefauthors}%
Sandewall, E.%
\end{APACrefauthors}%
\unskip\
\newblock
\APACrefYearMonthDay{1989}{}{}.
\newblock
{\BBOQ}\APACrefatitle {The semantics of non-monotonic entailment defined using
  partial interpretations} {The semantics of non-monotonic entailment defined
  using partial interpretations}.{\BBCQ}
\newblock
\BIn{} \APACrefbtitle {Non-monotonic Reasoning} {Non-monotonic reasoning}\
  (\BPGS\ 27--41).
\PrintBackRefs{\CurrentBib}

\bibitem [\protect \citeauthoryear {%
Shet%
}{%
Shet%
}{%
{\protect \APACyear {2007}}%
}]{%
shet2007bilatticeth}
\APACinsertmetastar {%
shet2007bilatticeth}%
\begin{APACrefauthors}%
Shet, V\BPBI D.%
\end{APACrefauthors}%
\unskip\
\newblock
\APACrefYear{2007}.
\unskip\
\newblock
\APACrefbtitle {Bilattice based logical reasoning for automated visual
  surveillance and other applications} {Bilattice based logical reasoning for
  automated visual surveillance and other applications}\ \APACtypeAddressSchool
  {\BUPhD}{}{}.
\unskip\
\newblock
\APACaddressSchool {}{Graduate school of the University of Maryland}.
\PrintBackRefs{\CurrentBib}

\bibitem [\protect \citeauthoryear {%
Shet%
, Harwood%
\BCBL {}\ \BBA {} Davis%
}{%
Shet%
\ \protect \BOthers {.}}{%
{\protect \APACyear {2006}}%
{\protect \APACexlab {{\protect \BCnt {1}}}}}]{%
shet2006multivalued}
\APACinsertmetastar {%
shet2006multivalued}%
\begin{APACrefauthors}%
Shet, V\BPBI D.%
, Harwood, D.%
\BCBL {}\ \BBA {} Davis, L\BPBI S.%
\end{APACrefauthors}%
\unskip\
\newblock
\APACrefYearMonthDay{2006{\protect \BCnt {1}}}{}{}.
\newblock
{\BBOQ}\APACrefatitle {Multivalued default logic for identity maintenance in
  visual surveillance} {Multivalued default logic for identity maintenance in
  visual surveillance}.{\BBCQ}
\newblock
\BIn{} \APACrefbtitle {European Conference on Computer Vision} {European
  conference on computer vision}\ (\BPGS\ 119--132).
\PrintBackRefs{\CurrentBib}

\bibitem [\protect \citeauthoryear {%
Shet%
, Harwood%
\BCBL {}\ \BBA {} Davis%
}{%
Shet%
\ \protect \BOthers {.}}{%
{\protect \APACyear {2006}}%
{\protect \APACexlab {{\protect \BCnt {2}}}}}]{%
shet2006top}
\APACinsertmetastar {%
shet2006top}%
\begin{APACrefauthors}%
Shet, V\BPBI D.%
, Harwood, D.%
\BCBL {}\ \BBA {} Davis, L\BPBI S.%
\end{APACrefauthors}%
\unskip\
\newblock
\APACrefYearMonthDay{2006{\protect \BCnt {2}}}{}{}.
\newblock
{\BBOQ}\APACrefatitle {Top-down, bottom-up multivalued default reasoning for
  identity maintenance} {Top-down, bottom-up multivalued default reasoning for
  identity maintenance}.{\BBCQ}
\newblock
\BIn{} \APACrefbtitle {Proceedings of the 4th ACM international workshop on
  Video surveillance and sensor networks} {Proceedings of the 4th acm
  international workshop on video surveillance and sensor networks}\ (\BPGS\
  79--86).
\PrintBackRefs{\CurrentBib}

\bibitem [\protect \citeauthoryear {%
Shet%
, Neumann%
, Ramesh%
\BCBL {}\ \BBA {} Davis%
}{%
Shet%
\ \protect \BOthers {.}}{%
{\protect \APACyear {2007}}%
}]{%
shet2007bilattice}
\APACinsertmetastar {%
shet2007bilattice}%
\begin{APACrefauthors}%
Shet, V\BPBI D.%
, Neumann, J.%
, Ramesh, V.%
\BCBL {}\ \BBA {} Davis, L\BPBI S.%
\end{APACrefauthors}%
\unskip\
\newblock
\APACrefYearMonthDay{2007}{}{}.
\newblock
{\BBOQ}\APACrefatitle {Bilattice-based logical reasoning for human detection}
  {Bilattice-based logical reasoning for human detection}.{\BBCQ}
\newblock
\BIn{} \APACrefbtitle {Computer Vision and Pattern Recognition, 2007. CVPR'07.
  IEEE Conference on} {Computer vision and pattern recognition, 2007. cvpr'07.
  ieee conference on}\ (\BPGS\ 1--8).
\PrintBackRefs{\CurrentBib}

\bibitem [\protect \citeauthoryear {%
Wilkes%
\ \protect \BOthers {.}}{%
Wilkes%
\ \protect \BOthers {.}}{%
{\protect \APACyear {2010}}%
}]{%
wilkes2010heterogeneous}
\APACinsertmetastar {%
wilkes2010heterogeneous}%
\begin{APACrefauthors}%
Wilkes, D\BPBI M.%
, Franklin, S.%
, Erdemir, E.%
, Gordon, S.%
, Strain, S.%
, Miller, K.%
\BCBL {}\ \BBA {} Kawamura, K.%
\end{APACrefauthors}%
\unskip\
\newblock
\APACrefYearMonthDay{2010}{}{}.
\newblock
{\BBOQ}\APACrefatitle {Heterogeneous artificial agents for Triage nurse
  assistance} {Heterogeneous artificial agents for triage nurse
  assistance}.{\BBCQ}
\newblock
\BIn{} \APACrefbtitle {Humanoid Robots (Humanoids), 2010 10th IEEE-RAS
  International Conference on} {Humanoid robots (humanoids), 2010 10th ieee-ras
  international conference on}\ (\BPGS\ 130--137).
\PrintBackRefs{\CurrentBib}

\end{thebibliography}

\end{document}